\documentclass[accepted,specialissue]{melba}

\usepackage{amsmath,amsfonts}

\usepackage{graphicx}
\usepackage{float}
\usepackage{booktabs}
\usepackage{amssymb}
\usepackage{subcaption}
\usepackage{bbm}

\newcommand{\TVx}{\mathrm{TV}(\mathbf{x})}

\newcommand{\btheta}{{\boldsymbol{\theta}}}

\newcommand{\by}{\mathbf{y}}
\newcommand{\bx}{\mathbf{x}}

\newcommand{\bu}{\mathbf{u}}

\newcommand{\bbarx}{\bar{\mathbf{x}}}

\newcommand{\TV}{\mathrm{TV}}
\newcommand{\bD}{{\boldsymbol{D}}}
\newcommand{\bV}{\boldsymbol{V}}
\newcommand{\bv}{\boldsymbol{v}}
\newcommand{\bk}{\boldsymbol{k}}
\newcommand{\tk}{\tilde{\boldsymbol{k}}}

\melbaid{2026:015}  
\doi{https://doi.org/10.59275/j.melba.2026-g41g}
\melbaauthors{Lin, Berggren and Löfstedt}  
\email{disi.lin@umu.se}
\volume{2026}
\firstpageno{297}  
\melbayear{2026}  
\datesubmitted{2026-01}  
\datepublished{2026-06}  

\melbaspecialissue{Uncertainty for Safe Utilization of Machine Learning in Medical Imaging (UNSURE) 2025}
\melbaspecialissueeditors{Mobarak Hoque, Raghav Mehta, Cheng Ouyang, Chen Qin, Marianne Rakic, Sandy Wells}

\ShortHeadings{MELBA Journal Sample Article}{Lin, Berggren and Löfstedt}

\title{Generalized TV--$\ell_p$ Structured Priors for \mbox{Bayesian $T_1$ Mapping}}

\author{
	\firstname Disi \surname Lin\aff{1}\orcid{0009-0001-9691-6042},
	\firstname Martin \surname Berggren\aff{1}\orcid{0000-0003-0473-3263},
    \firstname Tommy \surname Löfstedt\aff{1}\orcid{0000-0001-7119-7646}
}
\affiliations{
	\num 1 \addr Department of Computing Science, Umeå University, Sweden
}

\abstract{
	We propose an extended family of structured spatial priors that incorporates the total variation (TV) function with $\ell_p$ norms. 
    The prior is proven to be proper and incorporated into a Bayesian regression framework to enable uncertainty quantification in $T_1$ mapping, with posterior inference performed using the No-U-Turn Sampler (NUTS). This TV--$\ell_p$ construction is proven to constitute a well-defined family of prior distributions, and it naturally enforces spatial consistency and smooth variations in the estimated parameter maps.
    The method was evaluated in comparison to maximum-likelihood estimation and several Bayesian alternative priors based on the uniform, Gamma, and bounded TV priors. The evaluation includes experiments on synthetic brain and cardiac $T_1$ mapping datasets, as well as a real in-vivo breast $T_1$ mapping dataset.
    The results show that the TV--$\ell_p$ prior yields more concentrated posterior densities, indicating reduced uncertainty. It also consistently achieves lower variance and smaller (negative) bias, leading to more reliable estimates.
    Overall, embedding a TV-based structured penalty along with $\ell_p$ norms in a prior in a Bayesian model improves spatial coherence in $T_1$ maps and enhances uncertainty quantification, offering a robust approach for $T_1$ mapping with uncertainties.
	Our code is available at~\url{https://github.com/Disi2022/Proper-TV-Structured-Priors}.}


\begin{document}

\twocolumn[\maketitle]

\section{Introduction}
\enluminure{M}{agnetic} Resonance Imaging (MRI) is a widely used medical imaging modality to visualize a subject's anatomical structures and physiological processes. For example, cardiac MRI~\citep{rajiah2023cardiac} provides detailed images of the cardiovascular system while breast MRI is used to examine internal breast structures \citep{wekking2023breast}. Functional MRI (fMRI) can be used to measure brain activity by detecting changes in cerebral blood flow and oxygenation~\citep{deyoe1994functional}. 
MRI techniques play crucial roles in disease diagnosis and treatment.
Compared to other imaging modalities, such as ultrasound and computed tomography (CT), MRI provides superior soft tissue contrast and eliminates the risks of exposure to ionizing radiation.

Despite its advantages, conventional MRI primarily provides qualitative images based on contrast differences, which limits its sensitivity to subtle or early pathological changes~\citep{lavrova2021exploratory}. 
Additionally, MR image intensity values are not standardized and can vary substantially depending on factors such as manufacturer, sequence type, and acquisition parameters~\citep{carre2020standardization}.
Quantitative MRI (qMRI) is a variation of MRI that produces measurements in physical units rather than relying on contrast weighting~\citep{granziera2021quantitative}.
Techniques using qMRI have been developed to quantify biophysical parameters such as relaxation times ($T_1$, $T_2$, and $T_2^*$), proton density, diffusion, \textit{etc.}~\citep{jara2022primary}.
Obtaining parametric maps typically involves two steps:
1) Transform raw k-space data to the image domain, resulting in contrast-weighted images; and
2) fit the parametrically linked contrast-weighted images to a biophysical or MRI signal model, resulting in quantitative parametric maps~\citep{shafieizargar2023systematic}. 
Overall, qMRI provides a more objective and reproducible alternative to conventional MRI~\citep{hellstrom2025noise}.

In particular, $T_1$ mapping is a qMRI technique that measures the longitudinal or spin-lattice relaxation time. $T_1$ mapping has, \textit{e.g.}, been used to uncover pathological processes in the brain~\citep{grafe2021quantitative} and myocardial tissue characterization~\citep{becker2017multi}.
Multiple $T_1$ mapping pulse sequences have been proposed~\citep{tirkes2019evaluation}, with inversion recovery (IR) widely regarded as the gold standard due to its accuracy. IR records the recovery of longitudinal magnetization following an inversion pulse~\citep{pykett1983measurement}.
However, it requires a long acquisition time, which has prompted the development of various accelerated alternatives. One such alternative is saturation recovery, which replaces the inversion pulse with a saturation pulse.
Another prominent strategy is the Look--Locker (LL) method, which samples the recovery curve after inversion~\citep{look1970time}.
Building upon the LL framework, several refined variants have been proposed to improve efficiency and accuracy, including the Modified Look--Locker IR (MOLLI) sequence~\citep{messroghli2004modified} and the Modulated Repetition Time Look--Locker (MORTLL) method~\citep{gai2009modulated}.
Differing in magnetization preparation, the Variable Flip Angle (VFA) technique~\citep{christensen1974optimal} estimates $T_1$ from steady-state signals across multiple flip angles instead of recovery curves following inversion or saturation.

However, these $T_1$ mapping techniques are typically limited by long repetition times and introduce systematic errors in the $T_1$ estimation~\citep{preibisch2009influence}.
Undersampling the k-space, to sample below the Nyquist rate, has become the mainstream approach for accelerating MRI acquisition.
However, undersamping results in underdetermined equation systems, leading to artifacts when using linear reconstruction methods and lower signal-to-noise ratio in acquired data, which leads to poor reconstructed image quality~\citep{virtue2017empirical}.
Regularization, \textit{e.g.}, based on low-rank constraint and/or total variation (TV), has been used to solve such ill-posed inverse problems.
\citet{zhang2015accelerating} proposed a local low-rank constraint for VFA $T_1$ mapping, as well as multi-echo spin-echo $T_2$ mapping.
\citet{pandey2021joint} proposed a joint TV-based reconstruction for accelerated qMRI reconstruction.
\citet{le2022accelerated} used recurrent neural networks with a cyclic, model-based loss for cardiac MOLLI $T_1$ mapping.
\citet{wang2023free} proposed a model-based reconstruction method using a joint $\ell_1$--Wavelet spatial regularization and temporal TV regularization for free-breathing myocardial $T_1$ mapping.

The methods mentioned above are all frequentist approaches that maximize a likelihood (usually assuming Gaussian errors, corresponding to least-squares problems).
However, using (regularized) maximum likelihood estimation on each pixel independently typically leads to noisy reconstructions and does not account for uncertainty in the estimation. 
Bayesian inference has proven effective in modeling uncertainties by incorporating prior knowledge and observed data within a probabilistic framework~\citep{gelman1995bayesian}.
Bayesian models provide full posterior distributions over model parameters and predictions, allowing not only point estimates but also quantification of credibility and uncertainty.
For instance, \citet{beirinckx2022model} used a Bayesian approach with a TV prior for $T_1$ and $T_2$ mapping with super-resolution.
\citet{lofstedt2020bayesian} developed a Bayesian method based on a bounded TV prior for $T_1$ mapping that both reduced and estimated uncertainty in parametric maps.
\citet{huang2025accelerated} proposed a Bayesian formulation that models the wavelet coefficients of VFA-multi-echo images as Laplace distributed to achieve model-based parametric map reconstructions.

In Bayesian inference, the prior distribution encodes our beliefs about the unknown parameters before observing any data.
Common priors include noninformative priors, such as the uniform or Jeffreys priors, weak priors, and informative priors. 
Selecting an appropriate prior distribution is crucial since an inappropriate informative prior can disproportionately affect the resulting posterior inferences~\citep{jones2022quantifying}.
The Laplace prior is a commonly used prior that promotes sparsity but uses a single scale parameter to control both mass near zero and tail heaviness, often reducing flexibility and causing over-shrinkage of large coefficients.
To address this issue, alternative priors have been introduced that decouple sparsity enforcement from tail behavior.
Examples include the Horseshoe~\citep{carvalho2010horseshoe} or spike-and-slab priors~\citep{ishwaran2005spike}, both of which allow strong shrinkage of small coefficients while preserving large coefficients.
\citet{armagan2013generalized} proposed a generalized double Pareto prior which retains the sharp peak at zero characteristic of the Laplace distribution while exhibiting heavy, Student's t-like tails.
\citet{tang2019bayesian} adapted this prior for rank penalization under the Bayesian framework enabling adaptive learning of the rank.

While these priors effectively encode sparsity, they do not explicitly capture spatial structure. In images, neighboring pixels/voxels are typically strongly correlated, reflecting the inherent spatial continuity of physical tissues or structures. Incorporating such spatially structured information into the prior can therefore substantially improve inference quality.
TV regularization encourages smooth solutions by promoting sparsity in the spatial gradients of the coefficients, effectively favoring homogeneous parameter maps and sharp edges, and suppressing small, noisy fluctuations~\citep{michel2011total,hadj2018continuation}.
Incorporating TV directly into the Bayesian framework poses challenges: the standard TV penalty does not trivially correspond to a proper prior distribution, as it lacks a normalizable probability density over the parameter space (see Lemma \ref{lem:tv}).
To address this, \citet{lofstedt2020bayesian} used a bounded version of the TV prior, restricting the set of acceptable parameter values.
While that approach results in a proper prior, it still suffers from key limitations. 
It requires the exclusion of certain parameter values from the prior's support, limiting the flexibility and interpretability.
To overcome these issues, our previous work~\citep{disi2025} combined the anisotropic TV function and the $\ell_1$ norm, forming a proper prior, \textit{i.e.}, one that defines a valid probability distribution---without excluding any parameter values.
And experiments on synthetic brain $T_1$ mapping data showed that the method provides smoother results and narrower credible intervals.

In this paper, we rigorously show that a prior constructed solely from the TV functional cannot be normalized and therefore fails to define a proper probability distribution over the parameter space. We subsequently generalize the prior construction from previous work~\citep{disi2025} to accommodate both anisotropic and isotropic TV combined with any $\ell_p$ norm with $p \geq 1$, and prove that the resulting prior is proper.
We applied this proposed TV--$\ell_p$ prior to Bayesian $T_1$ mapping.

The proposed method was evaluated by comparing it with maximum-likelihood estimation and several Bayesian alternatives based on uniform, Gamma, and bounded TV priors. The evaluation includes experiments on synthetic brain and cardiac $T_1$ mapping datasets (acquired with VFA and MOLLI, respectively), as well as a real in-vivo breast $T_1$ mapping dataset (acquired with VFA).
The results show that across all datasets, the TV--$\ell_p$ prior consistently produces more concentrated posterior densities, indicating a substantial reduction in uncertainty compared to alternative methods. Moreover, it achieves lower variance and a smaller negative bias in the estimates, resulting in more robust and reliable $T_1$ estimates. These results highlight the ability of the TV--$\ell_p$ prior to preserve spatial structure while providing uncertainty quantification, making it particularly effective for both synthetic simulations and real in-vivo measurements.

\section{Background}

\subsection{Signal Models}

$T_1$ mapping is performed by fitting the measured data to a corresponding signal model. The proposed prior is evaluated using data acquired with the VFA and MOLLI sequences. The following subsection describes the signal models employed for these two acquisition schemes.

\subsubsection{VFA $T_1$ Mapping}

VFA $T_1$ mapping quantifies $T_1$ values from repeated scans at different excitation flip angles.
Specifically, VFA acquires two or more gradient-echo (GRE) data sets with different excitation angles, $\alpha_i$ for $i= 1, 2, \ldots, I$~\citep{baudrexel2018t1}.
The signal intensity, $S$, is a function of the longitudinal relaxation time, $T_1$, repetition time, TR, flip angle, $\alpha_i$, and the equilibrium longitudinal magnetization, $M_0$, which is determined by the proton density and other factors~\citep{deoni2003rapid}. 
The signal intensity is defined as
\begin{equation}
    \label{eq: vfa}
    S_{\btheta}(\alpha_i)
        = M_0 \frac{1-e^{-\frac{\text{TR}}{T_1}}}{1-\cos \left(\alpha_i\right) e^{-\frac{\text{TR}}{T_1}}} \sin \alpha_i,
\end{equation}
where $\btheta = (T_1, M_0)$ denotes the parameters to be estimated, which vary for tissues with different properties. In practice, however, residual unspoiled transverse magnetization may yield discrepancies between measured signal intensities and the theoretical values modeled in Equation~\eqref{eq: vfa}~\citep{baudrexel2018t1}. To account for this variability, the measurement at the $n$-th voxel is modeled as
\begin{equation}
    \label{eq:noise}
    y_n^{(i)} = S_{\btheta_n}^{(i)} + \epsilon_n^{(i)},
\end{equation}
where the additive noise terms, $\epsilon_n^{(i)}$, are assumed to be independent and follow a zero-mean Normal distribution with variance $\sigma^2$, \textit{i.e.}, $\epsilon_n^{(i)} \sim \mathcal{N}(0, \sigma^2)$. Here, $\btheta_n$ denotes the pair $(T_1, M_0)$ at voxel $n$, and the index $i$ identifies the acquisition associated with the $i$-th flip angle.

\subsubsection{MOLLI $T_1$ Mapping}

In contrast to VFA, the MOLLI method estimates $T_1$ using IR data. The MOLLI signal is modeled as
\begin{equation}\label{eq:MOLLI1}
    S_{\btheta}(t_i) = A - B \exp \left(- t_i / T_1^*\right),
\end{equation}
where $\btheta = (T_1^*, A, B)$ denotes the parameters to be estimated, which vary for tissues with different properties. 
As with VFA, the actual measurements are affected by systematic error~\citep{messroghli2004modified}. Therefore, the observed signal at voxel $n$ and inversion time $t_i$ is modeled following the same formulation as in Equation~\eqref{eq:noise},
The true longitudinal relaxation time, $T_1$, is derived from the estimated parameters as
\begin{equation}\label{eq:MOLLI2}
    T_1 = T_1^* \left( \frac{B}{A} - 1 \right).
\end{equation}

\subsection{MLE Reconstruction}

Typically, the parameters in $\btheta$ are estimated by fitting the model in Equation~\eqref{eq:noise} using (regularized) maximum log-likelihood estimation (MLE).
Assuming the measurements, $y_n^{(i)} \sim \mathcal{N}(S_{\btheta_n}^{(i)}, \sigma^2)$, are independent leads to a voxel-wise optimization problem that minimizes the least squares error, formulated as,
\begin{align}
    \hat{\btheta}_{\mathrm{ML}}
        &= \arg \max_{\btheta} \sum_{i=1}^I \sum_{n=1}^N \log \mathcal{N}\big( y^{(i)}_n \mid S_{\btheta_n}^{(i)}, \sigma^2 \big) \nonumber\\
        &= \arg \min_{\btheta} \left[ \frac{1}{2\sigma^2} \sum_{i=1}^I \sum_{n=1}^N \big( y^{(i)}_n - S_{\btheta_n}^{(i)} \big)^2 \right], \label{eq:likelihood} \\
\end{align}
where $N$ denotes the total number of voxels.

\subsection{Bayesian Model}

A Bayesian model combines a likelihood, such as that in Equation~\eqref{eq:likelihood}, with a prior distribution, that accounts for available prior information or any other prior beliefs.
Specifically, a Bayesian model assigns a posterior distribution over the parameters, constructed using a likelihood and a prior, as
\begin{equation} \label{eq:bayes}
    p(\btheta \mid \by) = \frac{p(\by \mid \btheta)p(\btheta)}{\int p(\by \mid \btheta)p(\btheta) d\btheta}
    \propto
    p(\by \mid \btheta)p(\btheta),
\end{equation}
where $p(\btheta \mid \by)$, $p(\by \mid \btheta)$, and $p(\btheta)$ are the posterior distribution, the likelihood, and the prior distribution, respectively.
The proportionality is in the parameters, $\btheta$. Since it is impossible or at least computationally intractable to compute the integral in the denominator, an efficient means is to draw samples from the posterior using Markov chain Monte Carlo (MCMC) methods~\citep{hoffman2014no}.

\subsection{Total Variation}

The TV of a smooth intensity function $\xi$ defined on a domain $\Omega$ is
\begin{equation*}
    \TV(\xi) = \int_\Omega\lVert\nabla\xi\rVert_2\,d\Omega.
\end{equation*}
For a two-dimensional (2D) image, where the intensity is represented by pixel values in an $n$-by-$n$ matrix, $\bx$, the gradient $\nabla\xi$ can be approximated with a finite difference scheme on the pixel values,  for instance, the forward differences  
\begin{equation}
    \nabla_{i,j}\bx = \bigr(x_{i+1,j} - x_{i,j}, x_{i,j+1}-x_{i,j}\bigl)^T.\label{eq:FD}
\end{equation}

There are two common versions of the TV of $\bx$, an isotropic and an anisotropic version.
Using the scheme in Equation~\eqref{eq:FD}, the isotropic TV is defined as~\citep{rudin1992nonlinear}
\begin{align}
    \TV_{iso}(\bx)
        &= \sum_{i,j} \|\nabla_{i,j} \bx\|_2 \nonumber\\
        &= \sum_{i,j} \sqrt{\left|x_{i+1, j}-x_{i, j}\right|^2+\left|x_{i, j+1}-x_{i, j}\right|^2} \label{eq:isotropic_TV},
\end{align}
and the anisotropic TV as~\citep{esedoḡlu2004decomposition}
\begin{align}
    \TV_{aniso}(\bx)
        &= \sum_{i,j} \|\nabla_{i,j} \bx\|_1 \nonumber\\
        &= \sum_{i,j} \left|x_{i+1, j}-x_{i, j}\right|+ \left|x_{i, j+1}-x_{i, j}\right|. \label{eq:anisotropic_TV}
\end{align}
Although here the argument, $\bx$, is assumed to be a 2D structure (\textit{e.g.}, a 2D image), the definitions generalize trivially to $d$-dimensional structures. 

Let $\bar{\bx} \in \mathbb{R}^N$ be the vectorized (flattened) version of $\bx$.
The isotropic TV can then be expressed as
$$
    \TV_{iso}(\bx) = \sum_{i,j} \|\bD_{i,j}\bar{\bx}\|_2,
$$
and the anisotropic TV  as
$$
    \TV_{aniso}(\bx) = \sum_{i,j} \|\bD_{i,j}\bar{\bx}\|_1 = \|\bD\bar{\bx}\|_1,
$$
where $\bD_{i,j} \in \{-1, 0, 1\}^{d \times N}$ denotes the elements of a finite-difference matrix (\textit{e.g.}, implementing the forward difference scheme in Equation~\eqref{eq:FD}) that computes discrete derivatives along $d$ dimensions from $\bbarx$ taking into account the spatial indices of $\bx$. 
The operator $\bD$ denotes the global linear difference operator obtained by stacking all local operators $\bD_{i,j}$, thereby computing all finite differences of $\bbarx$ simultaneously.

\section{Proposed Bayesian Model}

Choosing the prior that best represents available prior information is vital for Bayesian inference.
In this work, we propose to construct a spatially structured prior based on the TV function, which promotes smoothness by assigning high probabilities to small spatial changes in the parameter maps. The proposed formulation defines a proper prior density with full support over the parameter domain, ensuring that all possible values retain nonzero prior probability. We further formulate $T_1$ mapping within a Bayesian framework and employ the proposed prior to encode smooth spatial structure, thereby enabling uncertainty quantification and improving estimation performance.

\subsection{Structured Priors Based on Total Variation}

In images, neighboring pixels are often strongly correlated, and in qMRI images, neighboring pixels often represent the same underlying tissues. Hence, the parametric maps are likely to be locally spatially smooth.
In this section, we first naively construct a prior based on TV but then show that the resulting distribution is not integrable and hence does not induce a proper prior. To address this, we propose a novel family of priors that contains an additional energy term that ensures they are integrable.

A prior distribution based on the TV function can naively be defined using the Boltzmann distribution (Gibbs distribution) as
\begin{equation}
    \label{eq:tv_prior}
    p_\TV(\bx) = \frac{1}{Z} e^{-\lambda \TVx},
\end{equation}
where $\lambda > 0$, the $Z$ is a normalizing constant, and $\TV$ represent $\TV_{iso}$ or $\TV_{aniso}$.
This distribution results in high probabilities for sparse spatial changes in the model parameters.
However, the normalizing constant,
\begin{equation}
   Z = \int_{\mathbb{R}^N} e^{- \lambda\TVx} d\bbarx,
\end{equation}
diverges due to invariance in the direction of $\mathbbm{1}$, \textit{i.e.}, for constant images. We have the following result.
\begin{lemma} 
\label{lem:tv}
    Let $\lambda \in \mathbb{R}_{+}$, the integral
    $$
        Z = \int_{\mathbb{R}^N} e^{- \lambda\TVx} d\bbarx
    $$
    diverges for both isotropic and anisotropic TV.
\end{lemma}

\begin{proof}
Let $\bV = [\bv_1, \bv_2, \ldots, \bv_N]$ be an orthonormal basis in $\mathbb{R}^N$, where $\bv_1 = \frac{1}{\sqrt{N}}\mathbbm{1}$, where $\mathbbm{1} = (1,1,\dots,1)^T \in \mathbb{R}^N$ is the constant vector (with all elements being 1). Then $\bbarx = \bV \bk$, for some $\bk \in \mathbb{R}^N$.

For isotropic TV, $\bD_{i,j} \bv_1 = \mathbf{0}$ for any $i$ and $j$, thus,
\begin{align}
    \bD_{i,j}\bar{\bx}
        = \bD_{i,j} (\bV \bk)
        &= \bD_{i,j} \left( k_1 \bv_1 + \sum_{n=2}^N k_n \bv_n \right) \nonumber\\
        &= \bD_{i,j} \sum_{n=2}^N k_n \bv_n, \nonumber
\end{align}
which depends only on $\tilde{\bk} = (k_2, \ldots, k_N)^T$.
The integral for isotropic TV can thus be rewritten as
\begin{align}
    Z
        &= \int_{\mathbb{R}^N} e^{-\lambda\sum_{i,j} \|\bD_{i,j}\bar{\bx}\|_2} \, d\bbarx \nonumber\\
        &= \int_{\mathbb{R}^N}   e^{-\lambda\sum_{i,j}\left\| \bD_{i,j} \sum_{n=2}^N k_n \bv_n \right\|_2} \underbrace{|\det \bV|}_{=1}d\bk \label{eq:iso} \\
     &= \int_{-\infty}^\infty \left( \int_{\mathbb{R}^{N-1}} e^{-\lambda\sum_{i,j}\left\| \bD_{i,j} \sum_{n=2}^N k_n \bv_n \right\|_2} d \tk \right) dk_1. \nonumber
\end{align}

Let $C \;:=\; \int_{\mathbb{R}^{N-1}} e^{-\lambda\sum_{i,j}\left\| \bD_{i,j}\sum_{n=2}^N k_n\bv_n\right\|_2}\, d\tilde{\bk}$.
The integrand is strictly positive and bounded for all $\tilde{\bk}$, and therefore $C>0$. Moreover, since $\mathrm{null}(\TV) = \mathrm{span}\{\bv_1\}$, this guarantees exponential decay of the integrand at infinity and thus $C<\infty$.
Therefore,
\begin{align*}
 Z = \int_{-\infty}^{\infty} C \, dk_1 =
\lim_{R\to\infty} \int_{-R}^{R} C \, dk_1
= \lim_{R\to\infty} 2RC
= \infty.
\end{align*}

For anisotropic TV, $\bD \bv_1 = \mathbf{0}$, and thus,
\begin{equation*}
    \bD \bbarx = \bD (\bV \bk) = \bD \left( k_1 \bv_1 + \sum_{n=2}^N k_n \bv_n \right) = \sum_{n=2}^N k_n \bD \bv_n.
\end{equation*}
The integral for anisotropic TV can thus be rewritten as
\begin{align}
    Z &= \int_{\mathbb{R}^N} e^{-\lambda\|\bD\bbarx\|_1} \, d\bbarx \nonumber\\
      &= \int_{\mathbb{R}^N} e^{-\lambda\left\| \sum_{n=2}^N k_n \bD \bv_n \right\|_1} \underbrace{|\det \bV|}_{=1} d\bk \label{eq:aniso} \\
      &= \int_{-\infty}^\infty \left( \int_{\mathbb{R}^{N-1}} e^{-\lambda\left\| \sum_{n=2}^N k_n \bD \bv_n \right\|_1} d \tk \right) dk_1,
\end{align}

Let $C' \;:=\; \int_{\mathbb{R}^{N-1}} e^{-\lambda\left\| \sum_{n=2}^N k_n\bD\bv_n \right\|_1}\, d\tilde{\bk}$.
Again, $0 < C' < \infty$, and thus
\begin{equation*}
Z = \int_{-\infty}^{\infty} C' \, dk_1
=
\lim_{R\to\infty} \int_{-R}^{R} C' \, dk_1
= \lim_{R\to\infty} 2R C'
= \infty. 
\end{equation*}

Hence,
the normalizing constants diverge
for both $\TV=\TV_{iso}$ and $\TV=\TV_{aniso}$.
\end{proof}

In Lemma~\ref{lem:tv}, the normalization constant is shown to diverge when the integration is taken over the entire space $\mathbb{R}^N$, without imposing any constraints on the values of $\bar{\bx}$. 
However, in the practical quantitative MRI application of $T_1$ mapping considered in this work, the unknown parameter, $\bbarx$, represents $T_1$ relaxation times, and they are physically meaningful only for nonnegative values.
It is therefore natural to ask whether restricting the domain to the nonnegative orthant, $\mathbb{R}^N_+$, suffices to render the corresponding normalization constant finite.

The following Lemma~\ref{lem: tvx+} shows that this is not the case when the integration domain is restricted to $\mathbb{R}^N_+$, because the normalization constant associated with the $\TV$ prior remains divergent.
This result highlights that the divergence is not caused by allowing negative values, but rather is a direct consequence of the translation invariance of the TV semi-norm with respect to constant shifts, namely that $\TV(\bbarx) = \TV(\bbarx + c\,\mathbbm{1})$, $\; \forall\, c\in\mathbb{R}$.

\begin{lemma} 
\label{lem: tvx+}
Let $\bbarx \in \mathbb{R}^N_+$ and $\lambda \in \mathbb{R}_{+}$, the integral
    $$
        Z = \int_{\mathbb{R}_+^N} e^{- \lambda\TVx} d\bbarx
    $$
diverges for both isotropic and anisotropic TV.
\end{lemma}

\begin{proof}
     Fix $r>0$ and define the compact set
\begin{equation*}
    B_r := \{\bu\in\mathbb{R}^N : \|\bu\|_\infty \le r\}.
\end{equation*}
Since $\TV$ is continuous and $B_r$ is compact, there exists
\begin{equation*}
M_r := \max_{\bu\in B_r} \TV(\bu) < \infty.
\end{equation*}

Now define
\begin{equation*}
    S_r^s := \{\,\bbarx=t\mathbbm{1}+\bu : t \in [r,s),\ \bu\in B_r\,\}.
\end{equation*}
For any $\bbarx\in S_r^s$, each element thus satisfies $\bar{x}_k=t+u_k\ge r-r=0$ and  hence $S_r^s \subset \mathbb{R}^N_+$.
From Equation~\eqref{eq:isotropic_TV} and Equation~\eqref{eq:anisotropic_TV} it follows immediately that,
\begin{equation*}
    \TV(\bbarx)=\TV(t\mathbbm{1}+\bu)=\TV(\bu)\le M_r, \quad \forall\bbarx \in S_r^s,
\end{equation*}
and therefore
\begin{equation*}
    e^{-\lambda \TV(\bbarx)}\ge e^{-\lambda M_r}, \quad \forall \bbarx\in S_r^s.
\end{equation*}

Consequently, for any $s > r$
\begin{align*}
    Z &= \int_{\mathbb{R}_+^N} e^{-\lambda \TV(\bbarx)}\,d\bbarx
    \;\ge\;
    \int_{S_r^s} e^{-\lambda \TV(\bbarx)}\,d\bbarx\\
    &\;\ge\; e^{-\lambda M_r}\int_{S_r^s} d\bbarx
    \;=\; e^{-\lambda M_r}\,\mathrm{Vol}(S_r^s),
\end{align*}
where $\mathrm{Vol}(S_r^s)$ denotes the volume of the set $S_r^s$.
Taking the limit as $s\to\infty$ yields
\begin{equation*}
    Z \ge \lim_{s\to\infty} e^{-\lambda M_r}\,\mathrm{Vol}(S_r^s) = \infty.
\end{equation*}
\end{proof}

\citet{lofstedt2020bayesian} proposed to incorporate a uniform prior alongside the TV prior, such that the TV was only defined within the domain of the uniform prior, to force a proper prior on the model parameter and to ensure convergence during the sampling process.
This constraint bounds the model parameters to be in a certain range, making the corresponding integral finite.
Our previous work~\citep{disi2025} proposed an $\ell_1$-based modification of the anisotropic TV prior to ensure that it defines a proper probability distribution. Building on that formulation, the present work generalizes the prior by incorporating an $\ell_p$ norm with $p \geq 1$ into the isotropic or anisotropic TV function.
We propose the following modified TV function.

\begin{lemma} \label{lem: tvl1}
    Let $\mathrm{TV}_\mu^p$ be a modified TV function, such that
    \begin{equation}
       \mathrm{TV}_{\mu}^p(\bx)
            = \mathrm{TV}(\bx) + \mu\|\bbarx\|_p,
    \end{equation}
    where $\TV=\TV_{iso}$ or $\TV=\TV_{aniso}$, $\mu \in \mathbb{R}_{+}$, and $\|\cdot\|_p$ is an $\ell_p$ norm with $p\geq1$.
    Then the integral
    $$
        Z = \int_{\mathbb{R}^N} e^{-\lambda\mathrm{TV}_\mu^p(\bx)} \,d\bx < \infty,
    $$
    for all $\lambda \in \mathbb{R}_{+}$.
\end{lemma}

\begin{proof}
    Note that $\lambda \|\bD \bbarx\|_1 \geq 0$ for any $\bbarx \in \mathbb{R}^N$, so we can write
    \begin{align*}
        \int_{\mathbb{R}^N} e^{-\lambda\,\mathrm{TV}_\mu^p(\bx)} \, d\bbarx
        &= \int_{\mathbb{R}^N} e^{-\lambda \|\bD \bbarx\|_1 - \lambda \mu \|\bbarx\|_p} \, d\bbarx \\
        &\leq \int_{\mathbb{R}^N} e^{-\lambda \mu \|\bbarx\|_p} \, d\bbarx \\
        &\leq \int_{\mathbb{R}^N} e^{-\lambda \mu \|\bbarx\|_1 / N^{1 - \frac{1}{p}}} \, d\bbarx \\
        &= \prod_{n=1}^N \int_{-\infty}^{\infty} e^{-\lambda \mu |x_n| / N^{1 - \frac{1}{p}}} \, dx_n \\
        &= \left( \frac{2 N^{1 - \frac{1}{p}}}{\lambda \mu} \right)^N < \infty.
    \end{align*}
    The second inequality follows by noting that when $p=1$ then $N^{1-\frac{1}{p}}=1$, and that when $p>1$ Hölder's inequality gives \(\|\bx\|_1 \leq N^{1 - \frac{1}{p}} \|\bx\|_p\).
\end{proof}

A geometric illustration of Lemmas~\ref{lem:tv} and~\ref{lem: tvl1} is provided in Figure~\ref{fig:TVMU}.  
In Figure~\ref{fig:TVMU}(a), along the subspace spanned by \(\mathbbm{1}\), we have \(\|\bD \bbarx\|_1 = 0\), and thus the integrand \(e^{-\lambda \|\bD \bbarx\|_1}\) remains constant. As a result, integration along this unbounded direction leads to divergence of the integral. 
In contrast, Figures~\ref{fig:TVMU}(b) and \ref{fig:TVMU}(c) illustrate the integrand \(e^{-\mathrm{TV}_{\mu}^p(\bx)}\), where the inclusion of the additional \(\|\bbarx\|_p\) term induces exponential decay in all directions of \(\mathbb{R}^n\). This ensures that the integral converges and is finite.

\begin{figure*}[h!]
    \centering
    \includegraphics[height=0.3\linewidth]{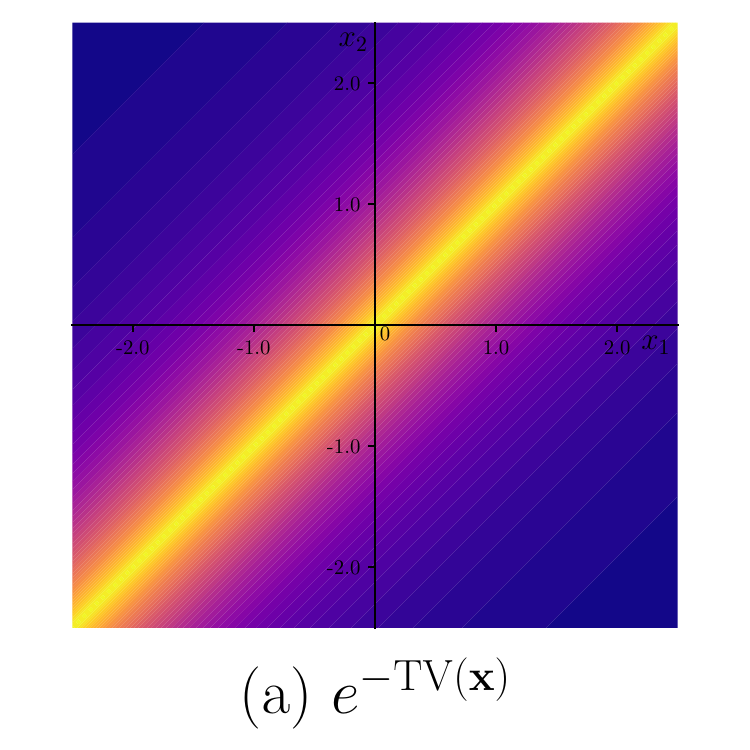}
    \includegraphics[height=0.3\linewidth]{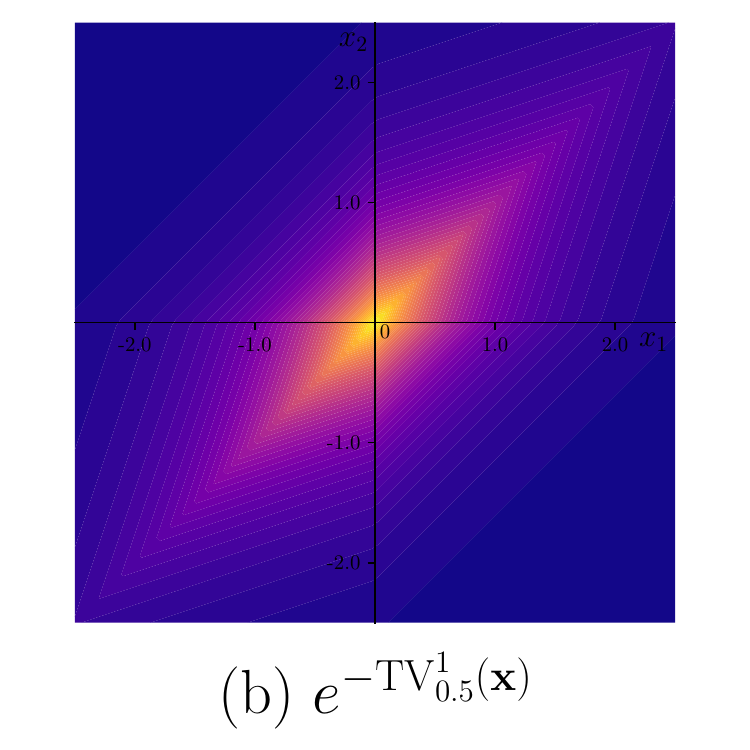}
    \includegraphics[height=0.3\linewidth]{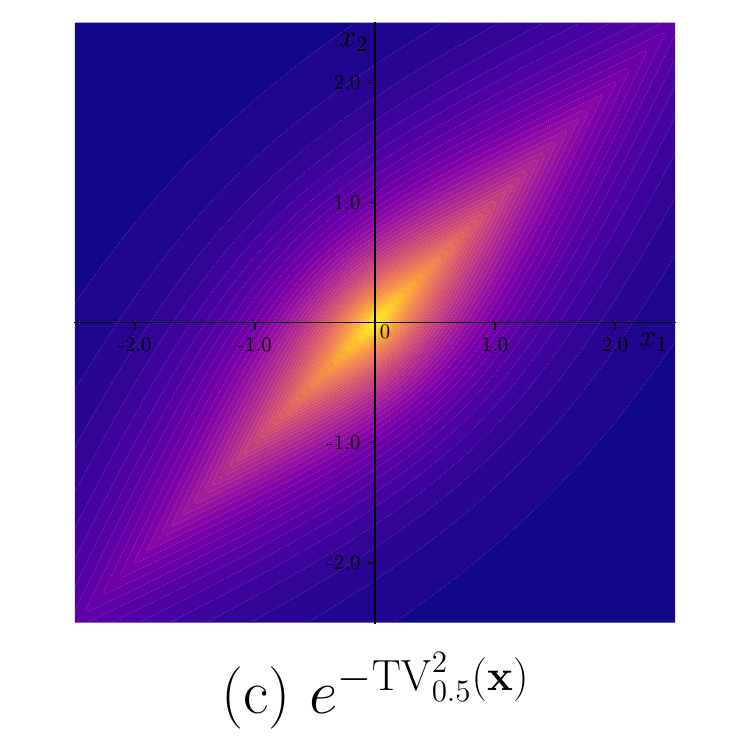}
    \raisebox{1.52em}{\includegraphics[height=0.2675\linewidth]{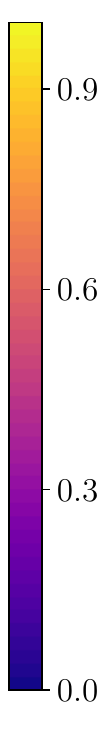}}
    \caption{2D heatmaps for the unnormalized priors (a): $e^{-\TV(\bx)}$; (b): $e^{-\TV_{0.5}^1(\bx)}$ and (c): $e^{-\TV_{0.5}^2(\bx)}$.}
    \label{fig:TVMU}
\end{figure*}

We see now that a finite normalization constant leads to a proper prior.
\begin{theorem}
    The modified prior,
    \begin{equation*}
        p_{TV_\mu^p}(\bx) = \frac{1}{Z}  e^{-\lambda \TV_\mu^p(\bx)},
    \end{equation*}
    is a proper prior.
\end{theorem}
\begin{proof}
    Follows from Lemma~\ref{lem: tvl1}, which says we can normalize it to unit integral.
\end{proof}

The generalized $\mathrm{TV}_{\mu}^p$ prior broadens the class of valid TV-based spatial priors and enhances modeling flexibility, allowing better adaptation to varying image characteristics and regularization requirements.

\subsection{Posterior Based on Total Variation}

We propose a Bayesian framework for $T_1$ mapping, where the likelihood is based on the signal models in Equations~\eqref{eq: vfa} or~\eqref{eq:MOLLI1}, and spatial prior information is incorporated through a generalized $\mathrm{TV}_\mu^p$ prior. A half-normal prior with scale parameter $\sigma_\epsilon > 0$ is placed on $\sigma$ to enforce positivity and encode prior knowledge about the noise level. The posterior is then given by,
\begin{align}
    &p_{TV_\mu^p}(\btheta,\sigma \mid \by)
        \propto p(\by \mid \btheta,\sigma) p_{TV_\mu^p}(\btheta)p(\sigma) \\
    &\;\; = \left(\sqrt{2 \pi \sigma^2} \right)^{-N I} \exp \left(-\frac{1}{2 \sigma^2}  \sum_{i=1}^I \sum_{n=1}^N\big(y_n^{(i)}-S_{\btheta_n}^{(i)}\big)^2\right) \nonumber\\
    &\qquad \cdot
        \prod_{\btheta} \frac{1}{Z_{\btheta}}  e^{-\lambda \TV_\mu^p(\btheta)} \cdot \frac{\sqrt{2}}{\sigma_\epsilon \sqrt{\pi}} \exp \left(-\frac{\sigma^2}{2 \sigma_\epsilon^2}\right), \nonumber
\end{align}
where for VFA acquisitions, $\btheta = (T_1,M_0)$, and for MOLLI acquisitions, $\btheta = (T_1^*, A, B)$. Since the overall normalizing constant is intractable, we sample from the unnormalized posterior using MCMC to quantify the uncertainty in the estimated parameters.

We further extend the Bayesian framework for $T_1$ mapping to a hierarchical model by assigning hyper-priors to the hyper-parameters $\lambda$ and $\mu$, which control the strength and shape of the generalized $\mathrm{TV}_\mu^p$ prior. The hierarchical model is defined as,
\begin{align}
    p_{TV_\mu^p} & (\btheta,\sigma, \lambda, \mu \mid \by) \\
        &\qquad \propto p(\by \mid \btheta,\sigma,\lambda,\mu) p_{TV_\mu^p}(\btheta \mid \mu, \lambda) p(\sigma) p(\lambda) p(\mu). \nonumber
\end{align}

\section{Experiments}

In the experiments, we evaluated the anisotropic $\TV_\mu^p$ prior with $p \in \{1,2\}$ for VFA and MOLLI $T_1$ mapping within the Bayesian models on both synthetic and real datasets. The performance was compared to that of MLE and alternative Bayesian models using uniform, Gamma, and bounded TV priors. The hyperparameters $\lambda$ and $\mu$ in the $\TV_\mu^p$ prior were selected either by assigning hyperpriors, referred to as the Hierarchical $\TV_\mu^p$ model, or through hyperparameter tuning, denoted simply as the $\TV_\mu^p$ model.

For the Hierarchical $\TV_\mu^p$ model, the $\lambda$ and $\mu$ are constrained to be non-negative and were therefore assigned exponential priors. The rate parameters of these exponential priors were determined using Bayesian optimization.
For the $\TV_\mu^p$ model, $\lambda$ and $\mu$ were treated as fixed hyperparameters and were tuned using Bayesian optimization.
During the Bayesian optimization, the optimal combination of hyperparameters was selected as the one that maximized the expected log point-wise predictive density~\citep[elpd;][]{vehtari2017practical}, estimated using the widely applicable information criterion \citep[WAIC;][]{watanabe2010asymptotic}, defined as,
\begin{align}
     &\text{WAIC} \approx \nonumber\\
     &\; \frac{1}{NI} \sum_{i=1}^I \sum_{n=1}^N \log \left(\frac{1}{S} \sum_{s=1}^{S} p\big(y^{(i)}_n \mid \btheta_n^s,\sigma^s\big)\right) \\
     &\quad\;\, - \frac{1}{NI} \sum_{i=1}^I \sum_{n=1}^N \frac{1}{S-1} \sum_{s=1}^{S}\left(\log p\big(y^{(i)}_n \mid \btheta_n^s,\sigma^s\big)-\mu^{(i)}_n\right)^2, \nonumber
\end{align}
where $\mu^{(i)}_n=\frac{1}{S} \sum_{s=1}^{S} \log p\big(y^{(i)}_n \mid \btheta_n^s, \sigma\big)$, and $\btheta_n^s$ and $\sigma^s$ with $ s=1,2,\dots,S$ denote samples drawn from the posterior distribution.
The parameter $p \in \{1,2\}$ in the $\TV_\mu^p$ model was also selected as the value that maximizes the elpd.

\subsection{Baselines}

MLE formulates the estimation as a voxel-wise optimization problem that minimizes the sum of squared errors, expressed as in Equation~\eqref{eq:likelihood}. The MLE solution was found using the L-BFGS-B algorithm~\citep{liu1989limited}.
The baseline priors used in the experiments were $\mathrm{Uniform}(0, 50)$, Gamma$(3,1)$ and the bounded TV prior.
Specifically, the $\mathrm{Uniform}(0, 50)$ prior represent weak prior knowledge. It enforces non-negativity and restricts the parameter to a physically meaningful range that covers all expected parameter values across tissues and acquisition settings. 
The Gamma$(3,1)$ prior was used to enforce non-negativity while providing moderate concentration around physiologically realistic $T_1$ values, offering more informative guidance than the $\mathrm{Uniform}(0, 50)$ prior.
The bounded TV prior was defined as~\citep{lofstedt2020bayesian}
\begin{equation*}
    p_{b\TV}(\bx) = p_{\TV}(\bx) \, p_{\mathrm{uni}}(\bx \mid L, H), \;\; \text{with} \;\; L = 0, \, H = 50,
\end{equation*}
where the uniform prior, \(p_{\mathrm{uni}}(\bx \mid L, H)\), imposes hard bounds on the parameter range.
The hyperparameter $\lambda$ was included in the bounded TV prior as well, and was selected through Bayesian optimization.

\subsection{Data}

The experimental evaluation was performed using datasets comprising both synthetic and real $T_1$ mapping data.
The synthetic brain $T_1$ mapping data were simulated using Equations~\ref{eq: vfa} and~\ref{eq:noise} with $TR = 6.8$ ms, $TE = 2.1$ ms, and flip angles $\alpha \in \{2^\circ, 4^\circ, 11^\circ, 13^\circ, 15^\circ\}$.
The ground truth values of $T_1$ and $M_0$ were generated from the BrainWeb phantom\footnote{https://brainweb.bic.mni.mcgill.ca/brainweb} using Hero Imaging\footnote{https://www.heroimaging.com}.
The parameter maps generated had a matrix size of $256 \times 256$ and a voxel size of $0.98 \times 0.98 \times 2.00$~mm$^3$.
The added noise was complex circular Gaussian noise.

For the synthetic cardiac $T_1$ mapping data, the ground-truth $T_1$ maps and inversion time were obtained from a publicly available dataset~\citep{le2022accelerated}. 
Specifically, the inversion times were
$$
    \mathrm{TI} \in [128,\, 208,\, 881,\, 951,\, 1628,\, 1711,\, 2381,\, 3131]~\mathrm{ms}.
$$
Synthetic measurements were generated using the MOLLI signal model described in 
Equations~\eqref{eq:MOLLI1} and~\eqref{eq:MOLLI2}, with additive Gaussian noise applied according to Equation~\eqref{eq:noise}.

We also used a real breast MRI dataset from the publicly available 
\textit{QIN-BREAST-02} collection\footnote{\href{https://www.cancerimagingarchive.net/collection/qin-breast-02}{https://www.cancerimagingarchive.net/collection/qin-breast-02}}~\citep{Yankeelov2019_QINBREAST2}. 
The data were acquired using a radio frequency-spoiled 3D GRE VFA technique with ten flip angles ranging from 
$2^\circ$ to $20^\circ$ in $2^\circ$ increments. The acquisition matrix was $192 \times 192 \times 20$, covering the full breast in the sagittal plane over a $22~\mathrm{cm}^2$ field of view, with a slice thickness of $5~\mathrm{mm}$.
From this dataset, we selected one subject (\texttt{QIN-BREAST-02-0011}) diagnosed with \textit{left-sided invasive ductal carcinoma}. 
The subject underwent neoadjuvant chemotherapy consisting of 
doxorubicin and cyclophosphamide, 
followed by paclitaxel. 

\subsection{Implementation details}

All Bayesian models were developed using PyMC 5.15.0\footnote{https://www.pymc.io}. 
Samples were drawn from the posterior distributions using the No-U-Turn Sampler algorithm~\citep[NUTS;][]{hoffman2014no}. Four independent Markov chains were run, each generating $\text{2,000}$ samples after discarding $\text{2,000}$ burn-in iterations. Convergence was evaluated by inspecting trace plots and calculating the $\hat{R}$ statistic, which assesses the consistency between and within chains to confirm reliable mixing and sampling~\citep{gelman1992inference}.

The results from the different methods were evaluated by analyzing the estimated posterior distributions, as well as the bias and variance of the posterior samples.
Specifically, the probability density functions (PDFs) of $T_1$ values were estimated using kernel density estimation with Gaussian kernels~\citep{silverman2018density}, based on posterior samples for each method. 
PDFs in representative regions were then visualized to examine their shape and spread, thereby providing insight into the estimation uncertainties.
The uncertainty was further quantified by computing the credible interval spanning two standard deviations around the posterior mean.
Bias and variance were also used for quantitative comparisons and were estimated as
\begin{equation}\label{eq:bias}
     \mathrm{Bias} = \frac{1}{S} \sum_{s=1}^{S} (\btheta^{s} - \btheta), \;\; \text{and}
\end{equation}
\begin{equation}\label{eq:variance}
    \mathrm{Varance} = \frac{1}{S-1} \sum_{s=1}^{S} \left( \btheta^{s} - \btheta_{\mathrm{mean}} \right)^2,
\end{equation}
where $\btheta$ denotes the ground truth, $\btheta^s$ is the $s$-th posterior sample, and $\btheta_{\mathrm{mean}} = \frac{1}{S} \sum_{s=1}^{S} \btheta^s$ is the empirical posterior mean.

\section{Results and Discussion}

This section presents the $T_1$ mapping results obtained on the synthetic Brain and Cardiac MRI datasets, as well as the real Breast MRI dataset, to compare all baseline and proposed methods.
The synthetic experiments provide controlled settings for quantitative evaluation against known ground truth, allowing us to investigate the influence of different priors. In contrast, the real-data experiments assess the practical applicability of the methods under realistic imaging conditions, where unknown complex noise are presented.

\subsection{Brain $T_1$ Mapping}

For the Brain $T_1$ Mapping, the $\TV_\mu^1$ model had the highest epld, so $p=1$ was selected. Figure~\ref{fig:pdf} shows the PDFs of the estimated $T_1$ values within the regions of interest (ROIs). The corresponding summary statistics are provided in Table~\ref{tab:brain}.
The bias and variance for each Bayesian method are visualized as spatial maps in Figure~\ref{fig:BiasMap}, alongside the mean $T_1$ maps, and as overall comparison metrics across all voxels in Figure~\ref{fig:BiasVar}.

\begin{figure*}[]
    \centering
    \includegraphics[height=0.2\linewidth]{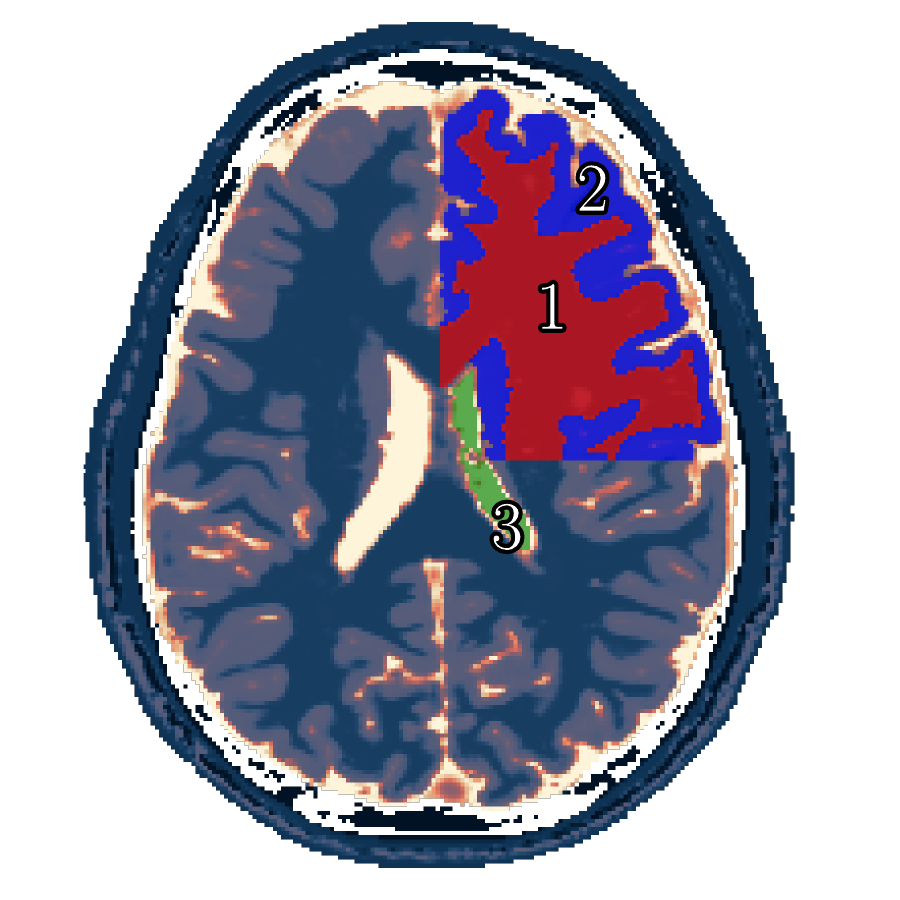}
    \includegraphics[height=0.2\linewidth]{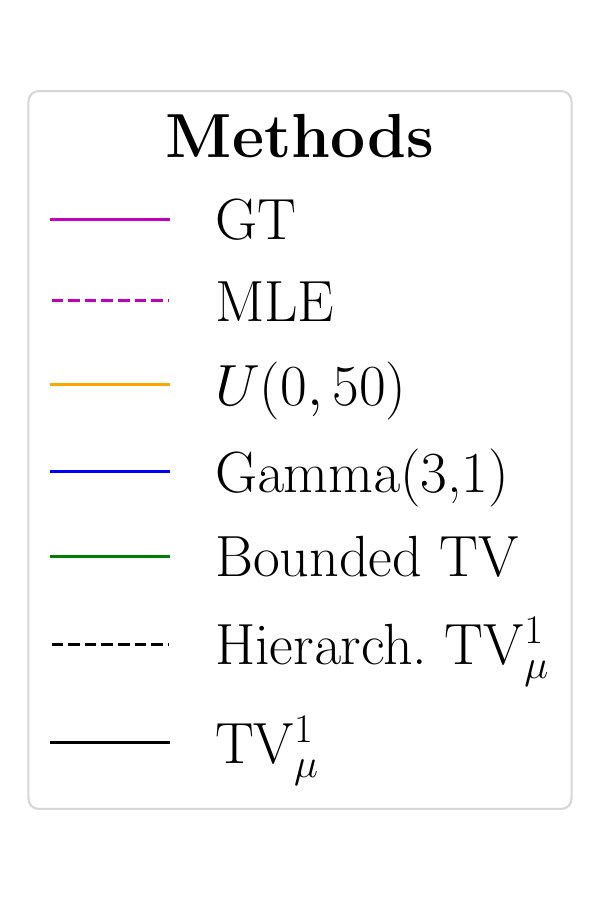}
    \includegraphics[height=0.2\linewidth]{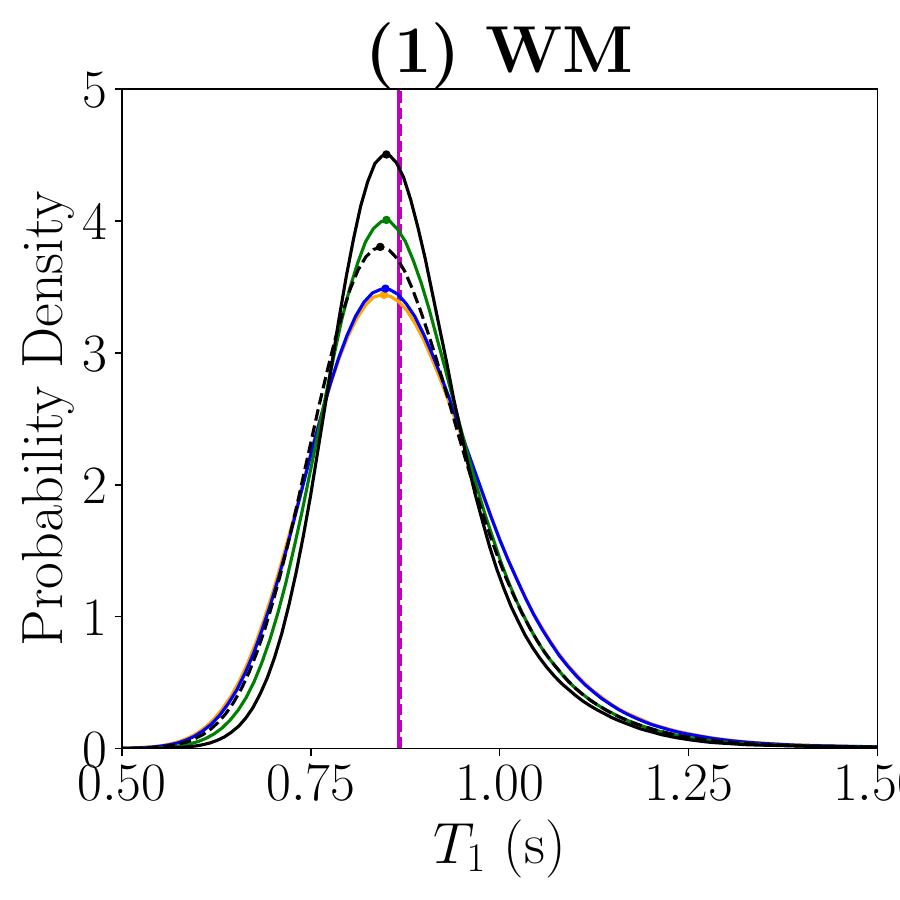}
    \includegraphics[height=0.2\linewidth]{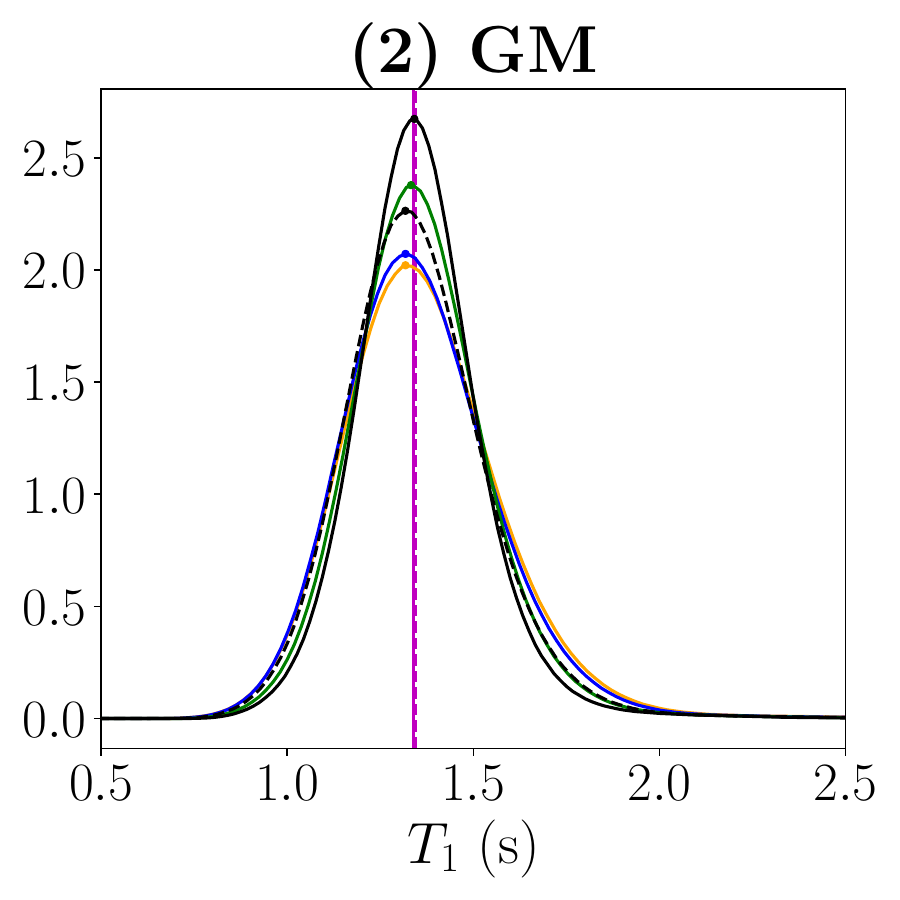}
    \includegraphics[height=0.2\linewidth]{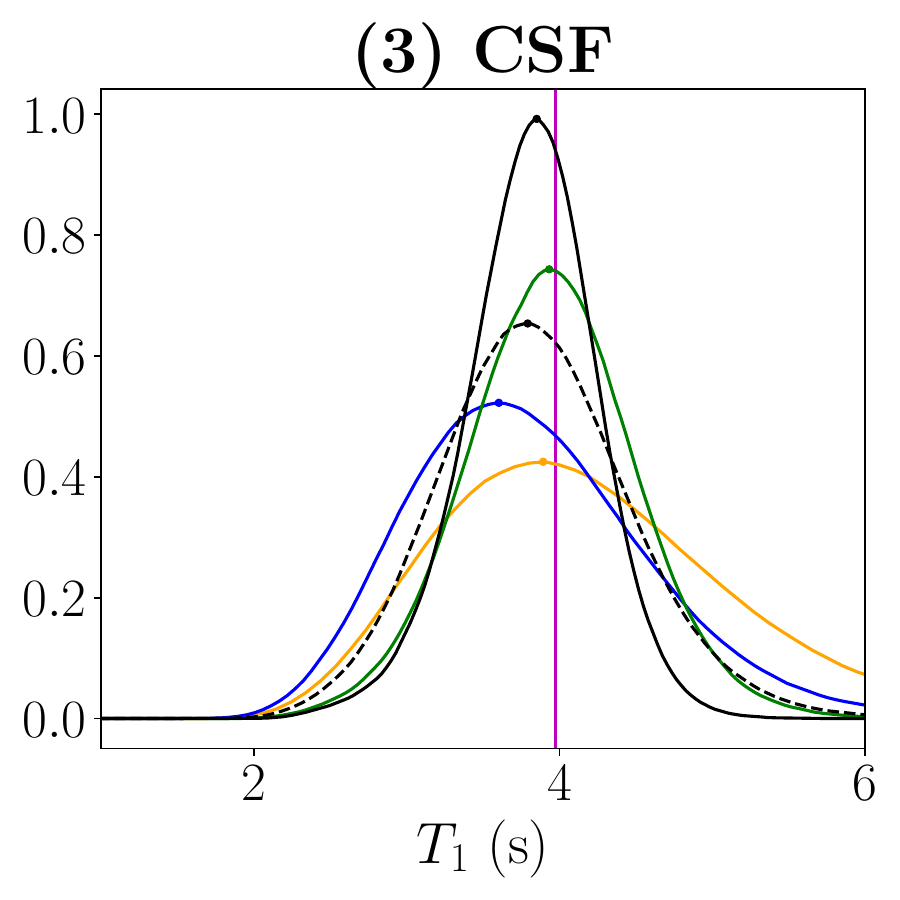}
    \caption{PDFs of estimated $T_1$ values within three selected ROIs. The ROIs correspond to (1) white matter (WM), (2) gray matter (GM), and (3) cerebrospinal fluid (CSF), highlighted in red, blue, and green, respectively, in the anatomical reference image. For each ROI, the PDFs from different Bayesian methods are shown, along with the corresponding ground truth (GT) and MLE $T_1$ values for comparison.}
    \label{fig:pdf}
\end{figure*}

\begin{table*}[]
\centering
\resizebox{0.97\textwidth}{!}{%
\begin{tabular}{lccccccc}
\toprule
\textbf{Tissue} & \textbf{GT} & \textbf{MLE} & $\mathbf{U(0,50)}$ & $\mathbf{Gamma(3,1)}$ & \textbf{Bounded}~$\boldsymbol{\TV}$ & \textbf{Hierarch.}~$\boldsymbol{\TV}^1_\mu$ & $\boldsymbol{\TV}^1_\mu$ \\
\midrule
\textbf{WM} & 0.87 & 0.87 & 0.85, 0.88 $\pm$ 0.26 & 0.85, 0.88 $\pm$ 0.26 & 0.85, 0.88 $\pm$ 0.23 & 0.84, 0.87 $\pm$ 0.24 & 0.85, 0.88 $\pm$ \textbf{0.21} \\
\textbf{GM} & 1.34 & 1.34 & 1.32, 1.37 $\pm$ 0.44 & 1.32, 1.36 $\pm$ 0.43 & 1.33, 1.36 $\pm$ 0.38 & 1.32, 1.35 $\pm$ 0.40 & 1.34, 1.36 $\pm$ \textbf{0.35}  \\
\textbf{CSF} & 3.97 & 32.01 & 3.89, 4.21 $\pm$ 2.07 & 3.60, 3.81 $\pm$ 1.59 & 3.93, 3.94 $\pm$ 1.12 & 3.79, 3.84 $\pm$ 1.26 & 3.85, 3.81 $\pm$ \textbf{0.86} \\
\bottomrule
\end{tabular}
}
\caption{Statistics of the estimated $T_1$ values in (1) WM, (2) GM, and (3) CSF across the different methods. For GT and MLE, the table reports the mean $T_1$ values. For each Bayesian model, as indicated in the top column, the table reports the posterior mode (\textit{i.e.}, the value with the highest probability density), the posterior mean, and the associated uncertainty expressed as two standard deviations around the mean, in the format ``$\mathrm{mode}, \mathrm{mean} \pm 2\sigma$''. All T$_1$ values are in seconds (s).}
\label{tab:brain}
\end{table*}

\begin{figure*}[]
    \centering
    \includegraphics[width=0.91\linewidth]{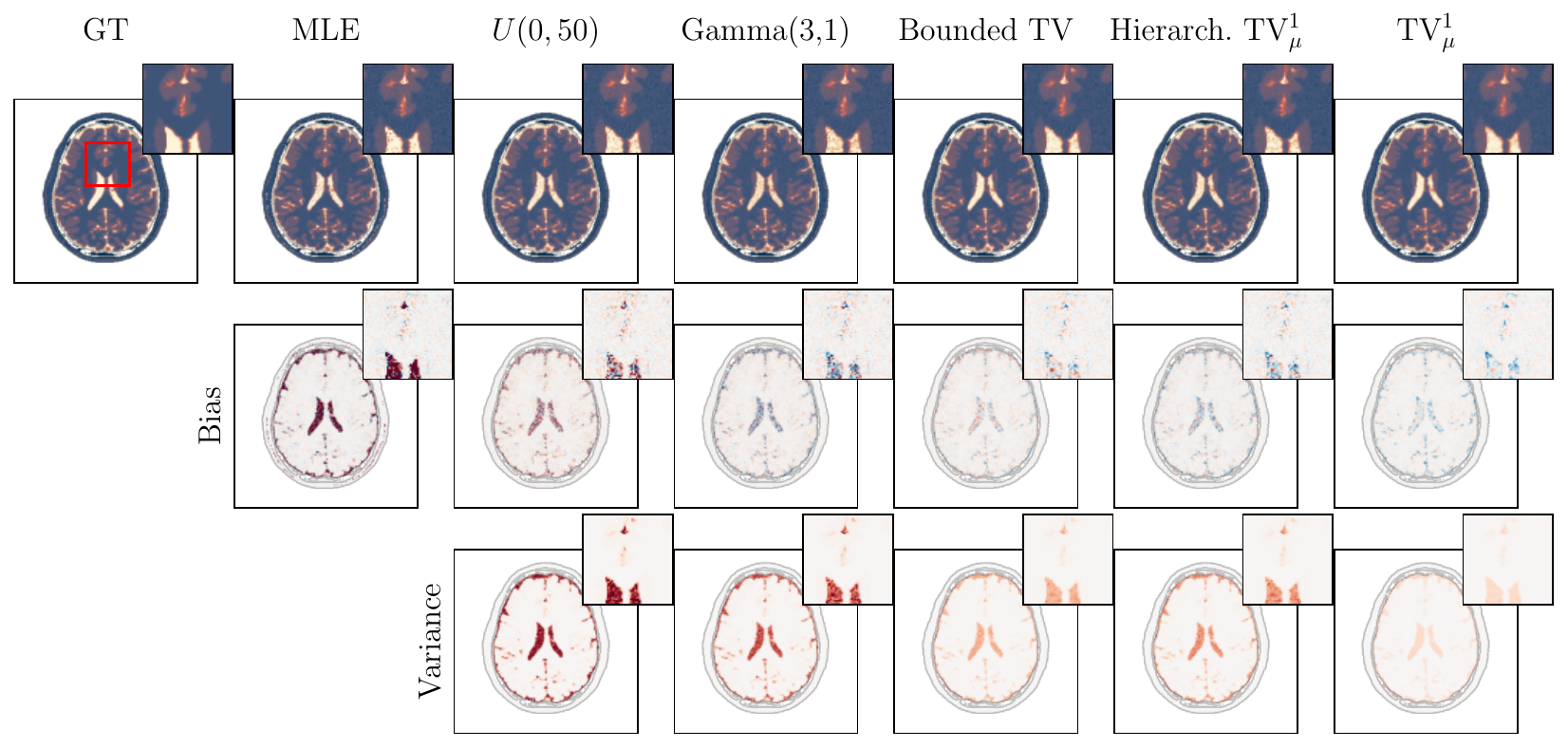}
    \includegraphics[height=0.395\linewidth]{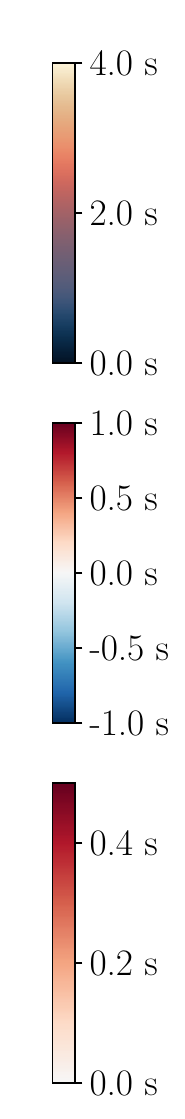}
    \caption{Comparison of mean $T_1$, bias and variance maps from different methods. The first row displays the GT $T_1$ map, the MLE point estimate, and the mean $T_1$ maps obtained from each Bayesian model, as labeled at the top of the figure. The second row presents the corresponding bias maps, while the third row shows the variance maps. Each subplot includes a zoomed-in view of a region (marked by the red square) in the top-right corner for detailed comparison.}
    \label{fig:BiasMap}
\end{figure*}

\begin{figure*}[]
    \centering
    \includegraphics[width=0.99\linewidth]{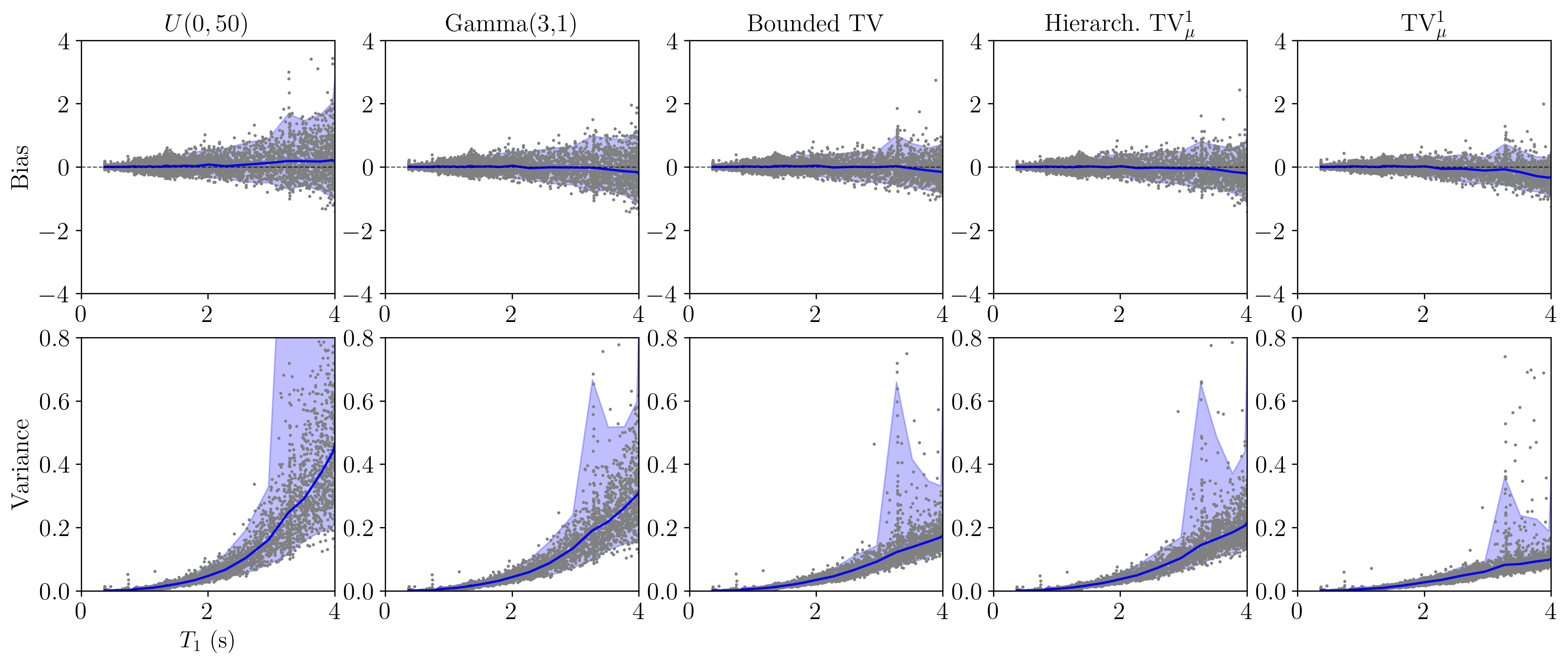}
    \caption{Bias and variance analysis for the Bayesian methods. In contrast to Figure~\ref{fig:BiasMap}, which presents spatially resolved maps, this figure illustrates the relationship between GT $T_1$ values and the corresponding bias or variance of the estimates. The horizontal axis shows the GT $T_1$ values, while the vertical axis indicates the bias/variance observed at each voxel. The blue curves represent the average bias/variance for a certain GT $T_1$ value. The blue shaded area contains 95\% of the data points.}
    \label{fig:BiasVar}
\end{figure*}

Figure~\ref{fig:pdf} illustrates that, across the three selected ROIs, the $\TV^1_\mu$ model produced the most concentrated posterior PDFs, followed by the Bounded TV and the Hierarchical $\TV^1_\mu$ models. In contrast, the Gamma and Uniform models result in progressively wider posterior distributions. This behaviour reflects a notable reduction in posterior uncertainty in the $T_1$ estimates achieved by the proposed method.
Table~\ref{tab:brain} shows that the standard deviations corresponding to the $\TV^1_\mu$ prior are smaller than all others across different tissues. 
The voxelwise MLE is particularly unstable in CSF, yielding unrealistically large $T_1$ estimates (\textit{e.g.}, 32.01s in Table~\ref{tab:brain}). This can be caused by the fact that, in CSF, the VFA signal is nearly insensitive to variations in $T_1$, resulting in a flat likelihood surface. Consequently, the inverse problem becomes weakly identifiable, allowing the MLE to drift toward excessively large values with minimal impact on the data fidelity term.

Figure~\ref{fig:BiasMap} shows that the average $T_1$ maps obtained using the $\TV^1_\mu$ method appear smoother than those obtained with other priors. This observation is consistent with the bias maps. Additionally, the variance maps indicate a lower variance when using the $\TV^1_\mu$ method.

Figure~\ref{fig:BiasVar} shows that the samples from the model using the $\TV^1_\mu$ prior have a smaller negative bias compared to other methods, especially when the $T_1$ values are high. Also, the variance from the proposed method is clearly much lower than for the other methods.

\subsection{Cardiac $T_1$ Mapping}

For the cardiac $T_1$ Mapping, the $\TV_\mu^p$ model with $p=1$ had the highest epld.
Figure~\ref{fig:pdfmolli} shows the PDFs of the estimated $T_1$ values within the selected ROIs. The corresponding summary statistics are provided in Table~\ref{tab:cardiac}.
The bias and variance for each Bayesian method are visualized as spatial maps in Figure~\ref{fig:BiasMapmolli}, alongside the mean $T_1$ maps, and as overall comparison metrics across all voxels in Figure~\ref{fig:BiasVarmolli}.

\begin{figure*}[]
    \centering
    \includegraphics[height=0.2\linewidth]{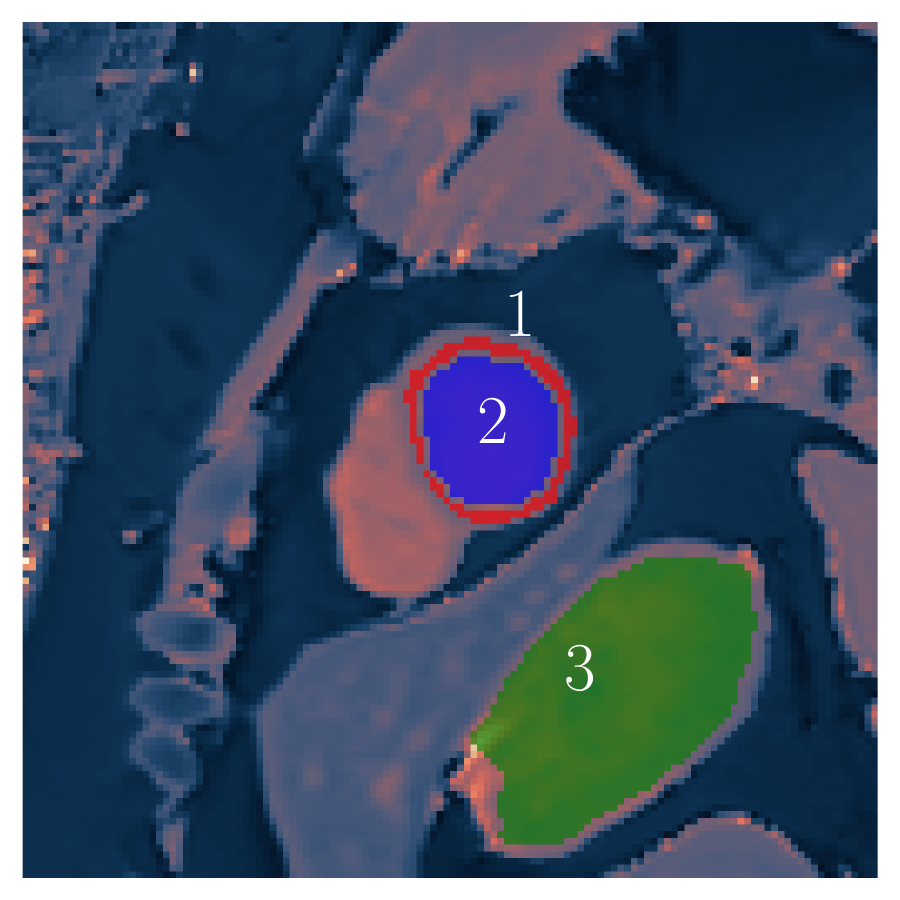}
    \includegraphics[height=0.2\linewidth]{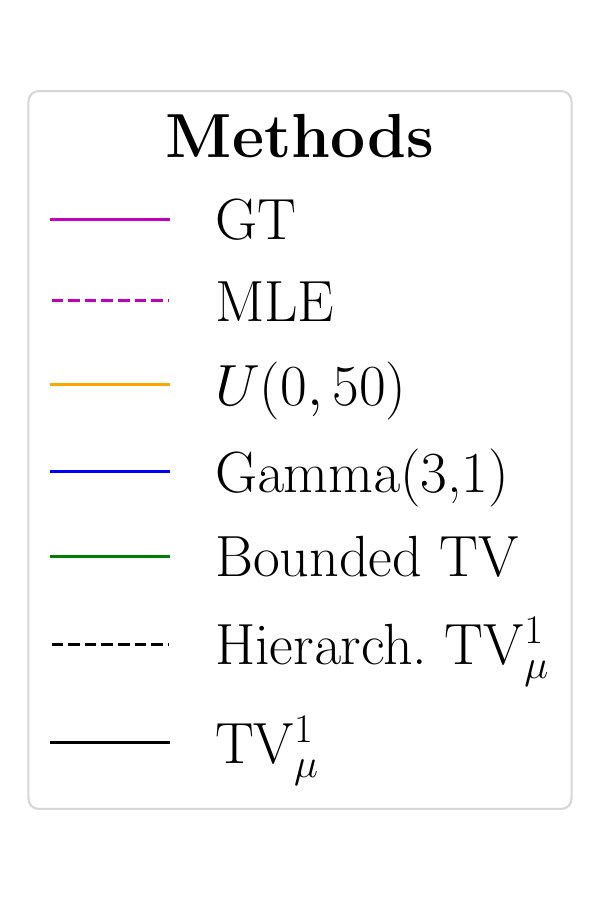}
    \includegraphics[height=0.2\linewidth]{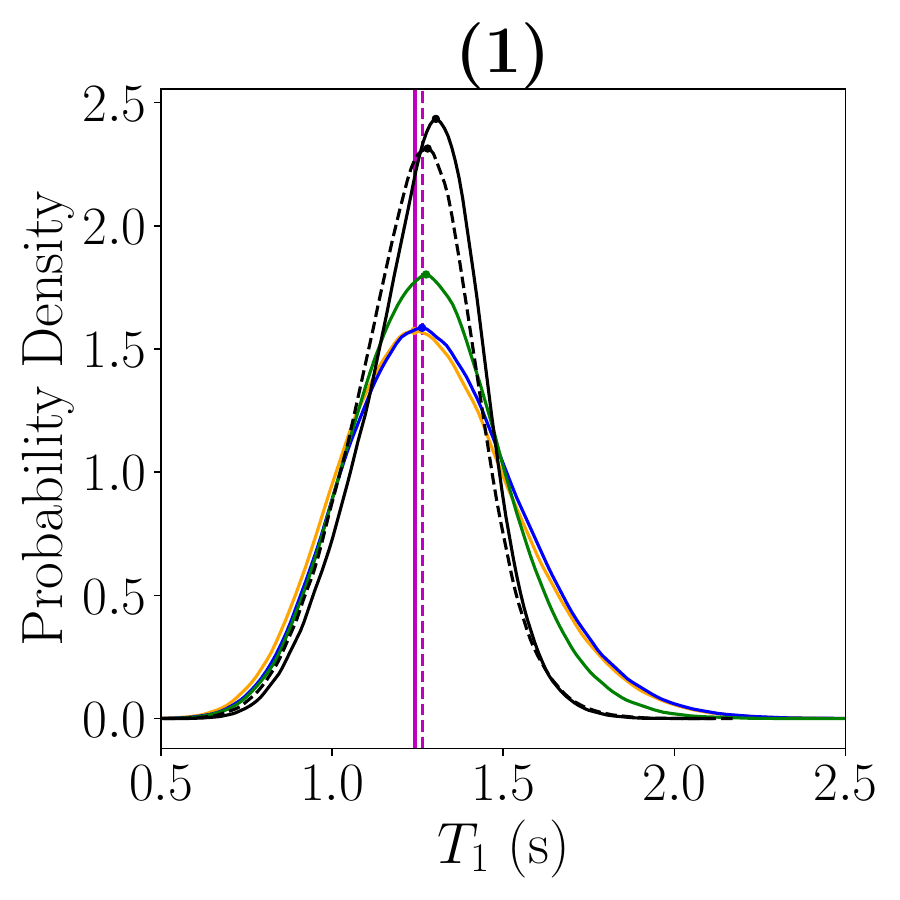}
    \includegraphics[height=0.2\linewidth]{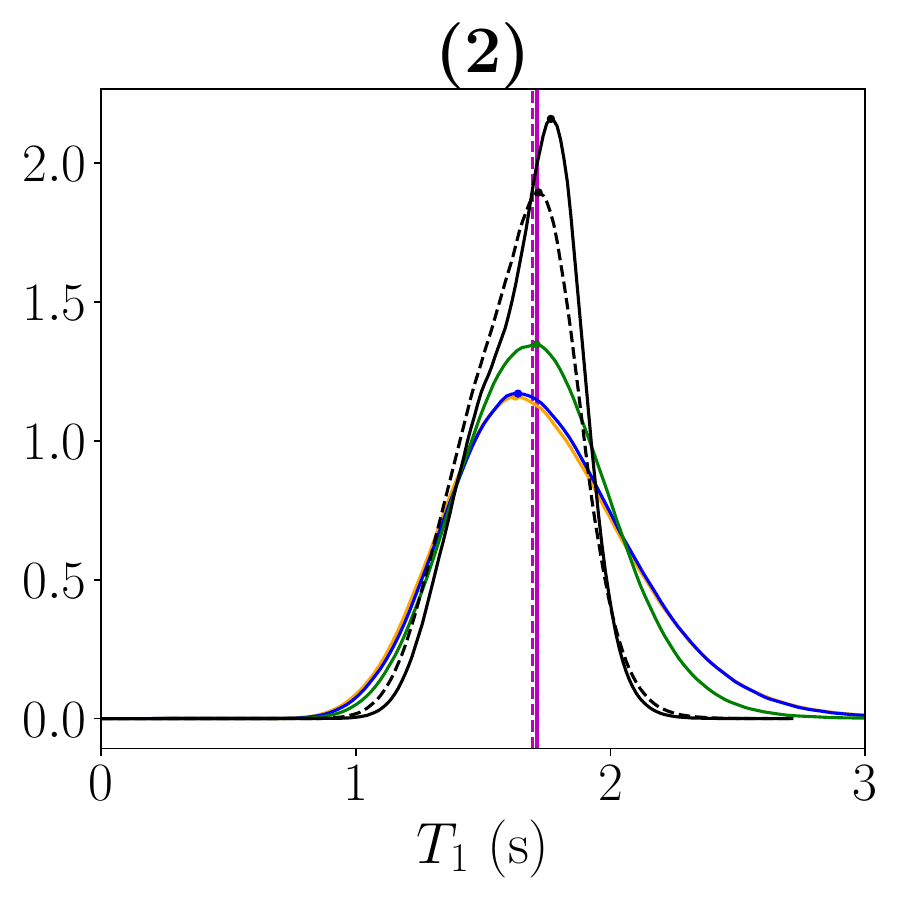}
    \includegraphics[height=0.2\linewidth]{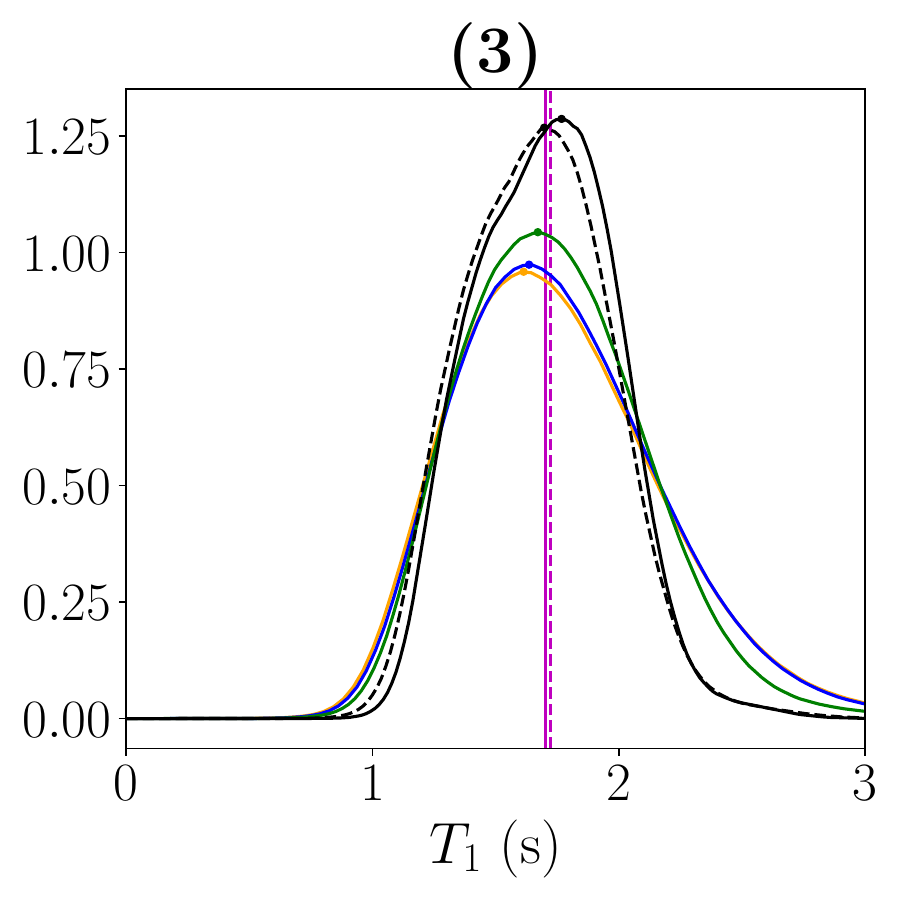}
    \caption{PDFs of $T_1$ values at selected ROIs (highlighted in red associated with myocardium, blue, and green on the anatomical image). The estimated PDFs from different Bayesian methods are shown, along with the GT and MLE $T_1$ values for comparison.}
    \label{fig:pdfmolli}
\end{figure*}

\begin{table*}[]
    \centering
    \resizebox{0.97\textwidth}{!}{%
        \begin{tabular}{lccccccc}
            \toprule
            \textbf{Tissue} & \textbf{GT} & \textbf{MLE} & $\mathbf{U(0,50)}$ & $\mathbf{Gamma(3,1)}$ & \textbf{Bounded} $\boldsymbol{\TV}$ & \textbf{Hierarch.}~$\boldsymbol{\TV}^1_\mu$ & ${\boldsymbol{\TV}^1_\mu}$ \\
            \midrule
            \textbf{(1)} & 1.24 & 1.26 & 1.24, 1.29 $\pm$ 0.51 & 1.26, 1.30 $\pm$ 0.51 & 1.27, 1.28 $\pm$ 0.45 & 1.28, 1.24 $\pm$ 0.36 & 1.30, 1.26 $\pm$ \textbf{0.35} \\
            \textbf{(2)} & 1.71 & 1.69 & 1.63, 1.73 $\pm$ 0.72 & 1.64, 1.74 $\pm$ 0.71 & 1.71, 1.71 $\pm$ 0.59 & 1.72, 1.64 $\pm$ 0.43 & 1.77, 1.67 $\pm$ \textbf{0.40} \\
            \textbf{(3)} & 1.70 & 1.72 & 1.61, 1.76 $\pm$ 0.88 & 1.64, 1.77 $\pm$ 0.86 & 1.67, 1.74 $\pm$ 0.76 & 1.70, 1.67 $\pm$ 0.59 & 1.77, 1.70 $\pm$ \textbf{0.57} \\
            \bottomrule
            \end{tabular}
    }
    \caption{Statistics of the estimated $T_1$ values in ROIs across the different methods. The table lists the GT $T_1$ values and the MLEs, alongside the results from different Bayesian methods, as indicated in the top column. For each Bayesian model, the table reports the posterior mode, mean, and two standard deviations as ``$\mathrm{mode}, \mathrm{mean} \pm 2\sigma$''. All $T_1$ values are reported in seconds (s).}
    
    \label{tab:cardiac}
    \end{table*}

\begin{figure*}[]
    \centering
    \includegraphics[width=0.9\linewidth]{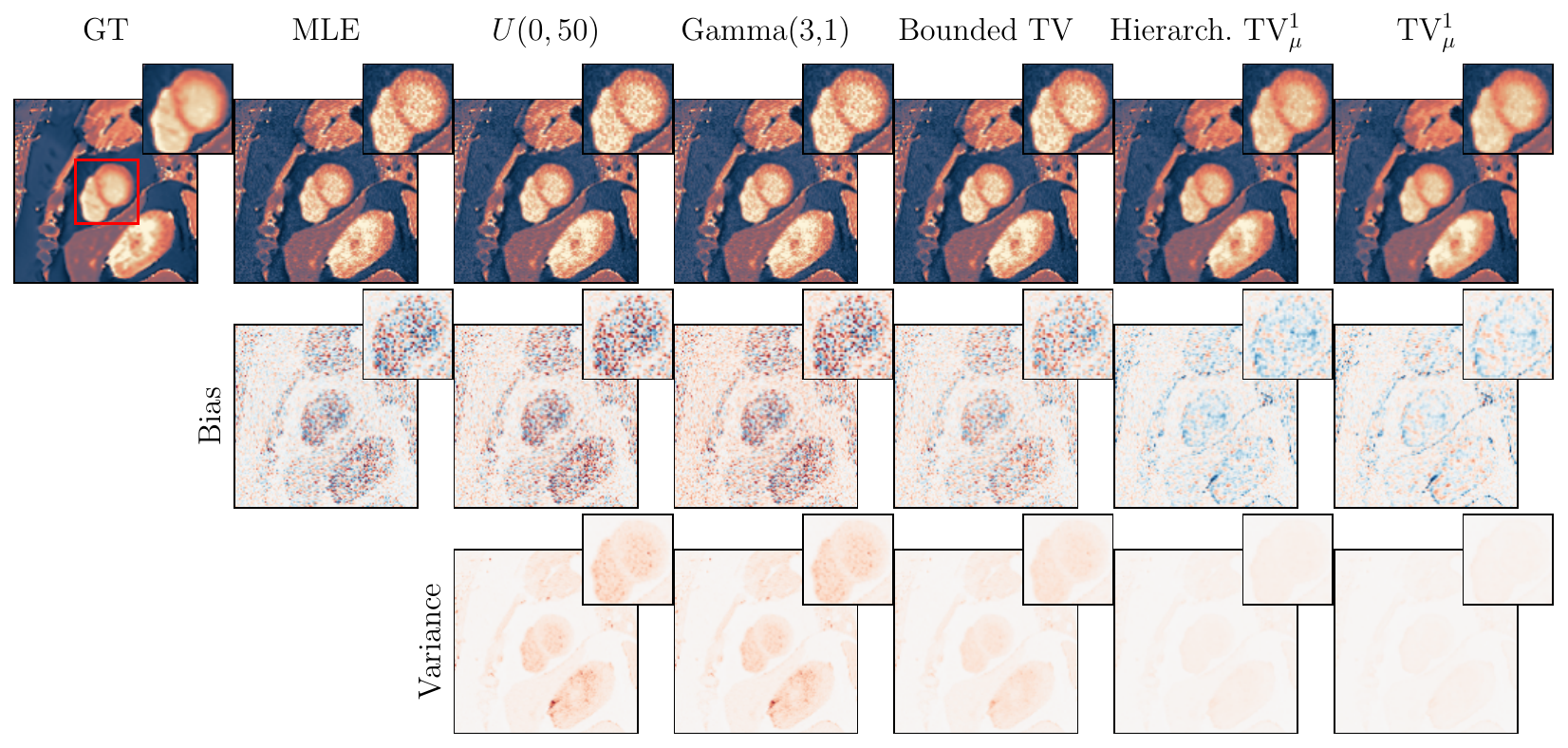}
    \includegraphics[height=0.4\linewidth]{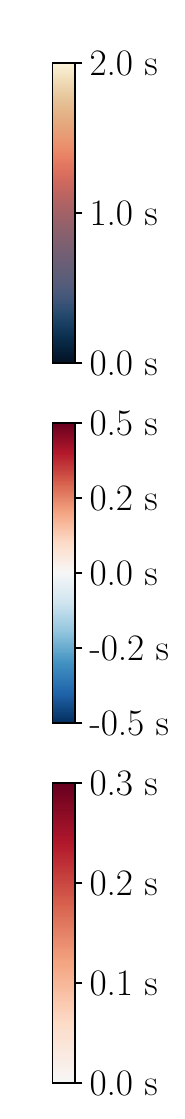}
    \caption{
    Comparison of mean $T_1$, bias, and variance maps from different methods. The first row displays the GT $T_1$ map, the MLE point estimate, and the mean $T_1$ maps obtained from each Bayesian model, as labeled at the top of the figure. The second row presents the corresponding bias maps, while the third row shows the variance maps. Each subplot includes a zoomed-in view of a region (marked by the red square) in the top-right corner for detailed comparison.
    }
    \label{fig:BiasMapmolli}
\end{figure*}

\begin{figure*}[]
    \centering
    \includegraphics[width=0.99\linewidth]{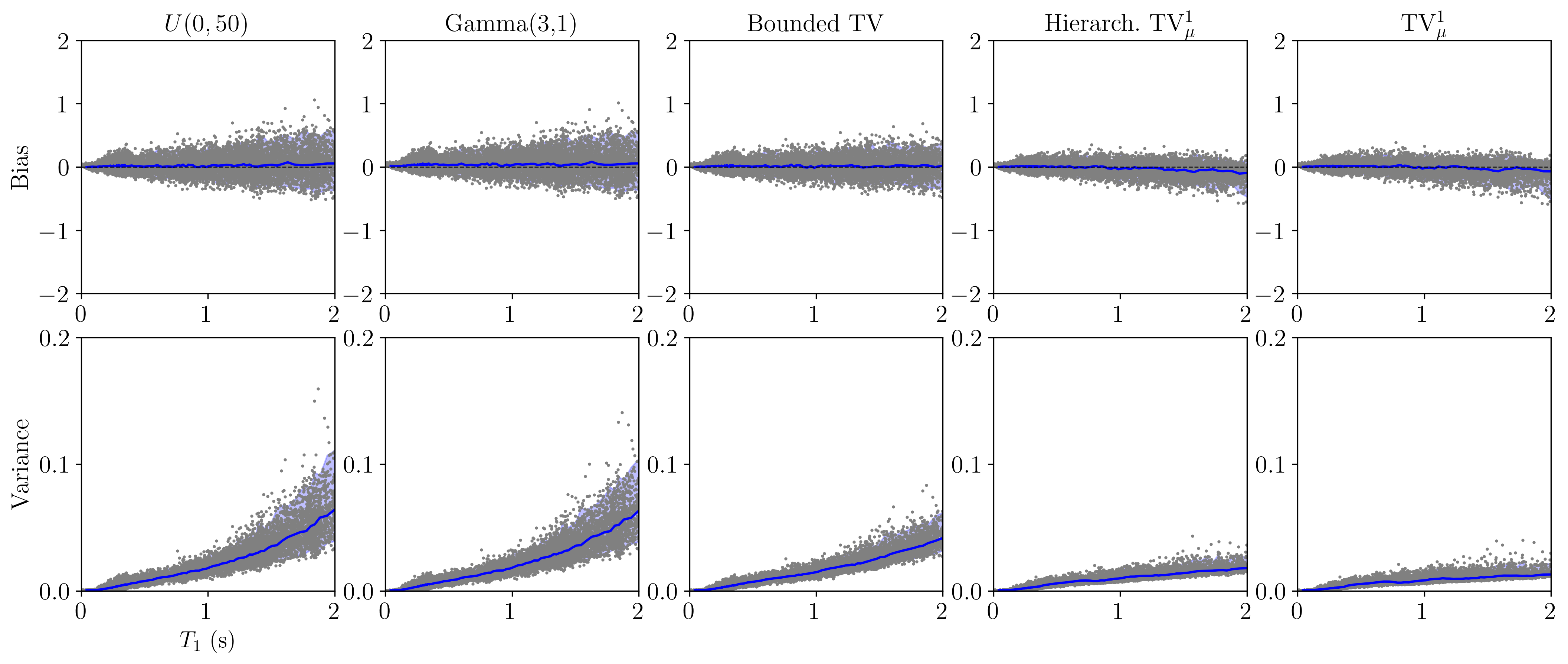}
    \caption{Bias and variance analysis for the Bayesian methods. The horizontal axis shows the GT $T_1$ values, while the vertical axis indicates the bias/variance observed at each voxel. The blue curves represent the average bias/variance for a certain GT $T_1$ value. The blue shaded area contains 95\% of the data points.}
    \label{fig:BiasVarmolli}
\end{figure*}

Figure~\ref{fig:pdfmolli} illustrates that, across the three selected ROIs, the $\TV^1_\mu$ method produces the most concentrated posterior PDFs, whereas the Hierarchical $\TV^1_\mu$, Bounded TV, Gamma, and Uniform models yield progressively wider distributions. This behaviour reflects a notable reduction in posterior uncertainty in the $T_1$ estimates achieved by the proposed $\TV^1_\mu$ method.

Table~\ref{tab:cardiac} shows that the standard deviations corresponding to the $\TV^1_\mu$ prior are smaller than all others across different tissues.

Figure~\ref{fig:BiasMapmolli} shows that the average $T_1$ maps obtained using the $\TV^1_\mu$ method appear smoother than those obtained with other priors. This observation is consistent with the bias maps. Additionally, the variance maps indicate a lower variance when using the $\TV^1_\mu$ method.

Figure~\ref{fig:BiasVarmolli} shows that the samples from the model using the $\TV^1_\mu$ prior have a smaller negative bias compared to other methods, especially when the $T_1$ values are high. Also, the variance from the proposed method is clearly much lower than for the other methods.

\subsection{Breast $T_1$ Mapping}

Similar to the synthetic experiment, the $\TV_\mu^p$ model with $p=1$ had the highest epld for the Breast $T_1$ Mapping. 
Figure~\ref{fig:pdfbreast} presents the PDFs of the estimated $T_1$ values within the ROIs. 
Figure~\ref{fig:pdfbreast} shows that the Hierarchical $\TV^1_\mu$ method produces the most concentrated posterior PDFs, whereas the $\TV^1_\mu$, Bounded TV, Gamma, and Uniform models yield progressively wider distributions. The multimodal posterior distributions observed for ROIs 2, 3, and 6 reflect heterogeneous tissue composition within the analyzed regions.
The corresponding summary statistics are provided in Table~\ref{tab:breast}.
Table~\ref{tab:breast} shows that the standard deviations corresponding to the Hierarchical $\TV^1_\mu$ prior are smaller than all others across different tissues.

Figure~\ref{fig:mapssbreast} shows the estimated $T_1$ map from MLE and the estimated mean $T_1$ maps from Uniform, Gamma, bounded $\TV$, and proposed Hierarchical $\TV_\mu^1$ and $\TV_\mu^1$ models. The estimated $T_1$ maps from the proposed models are smoother than those from classical Bayesian alternatives.

\begin{figure*}[]
    \centering
    \includegraphics[height=0.2\linewidth]{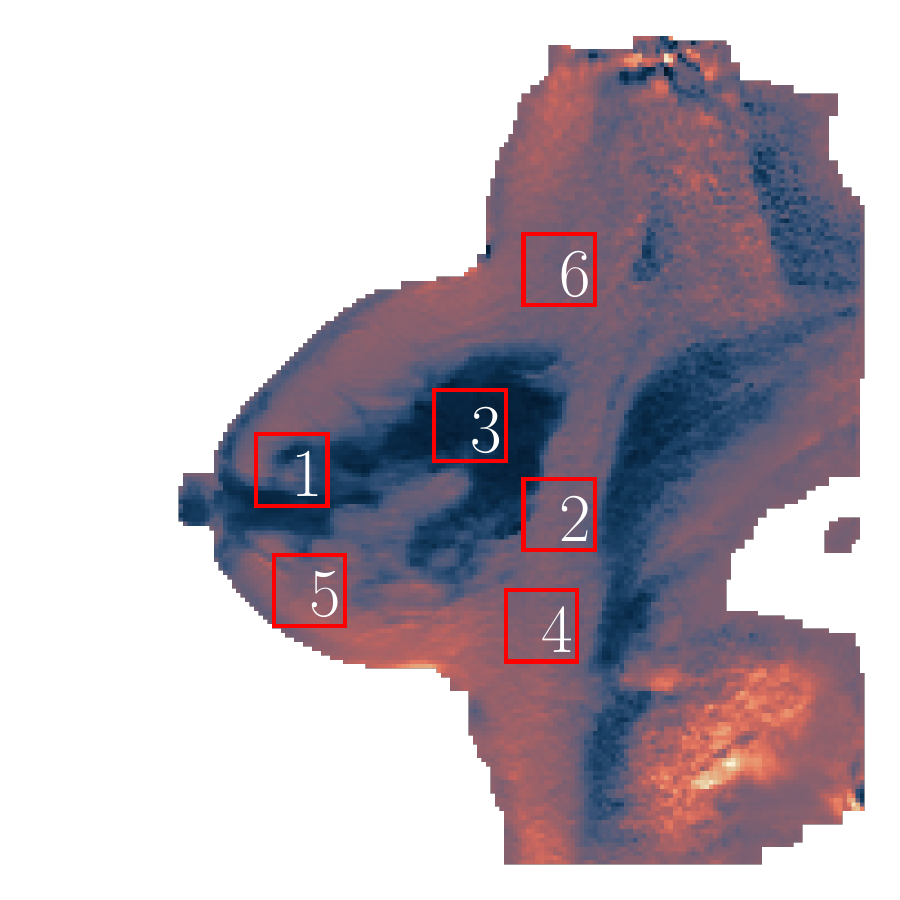}
    \includegraphics[height=0.2\linewidth]{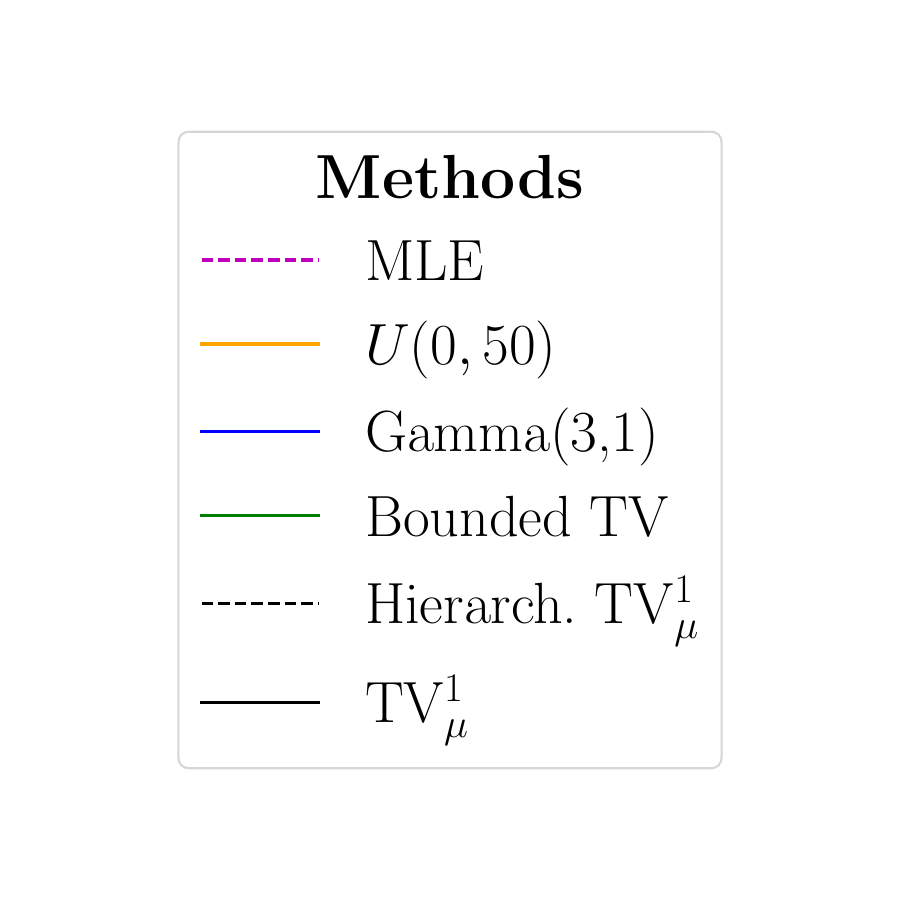}
    \includegraphics[height=0.2\linewidth]{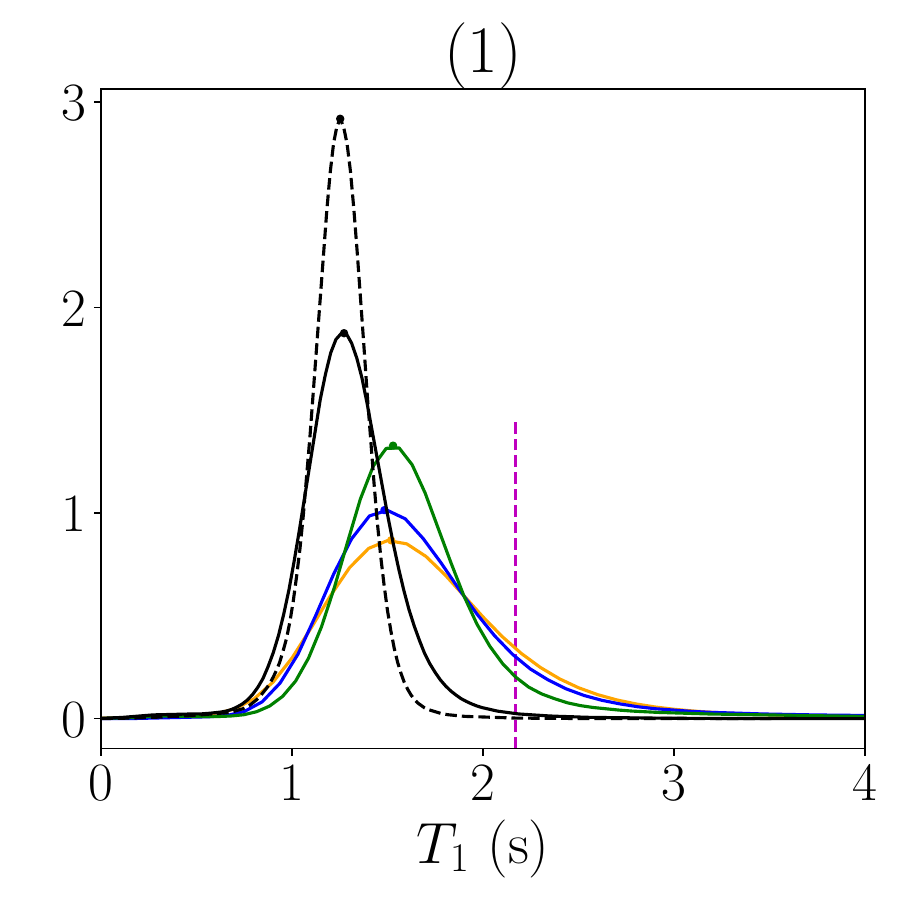}
    \includegraphics[height=0.2\linewidth]{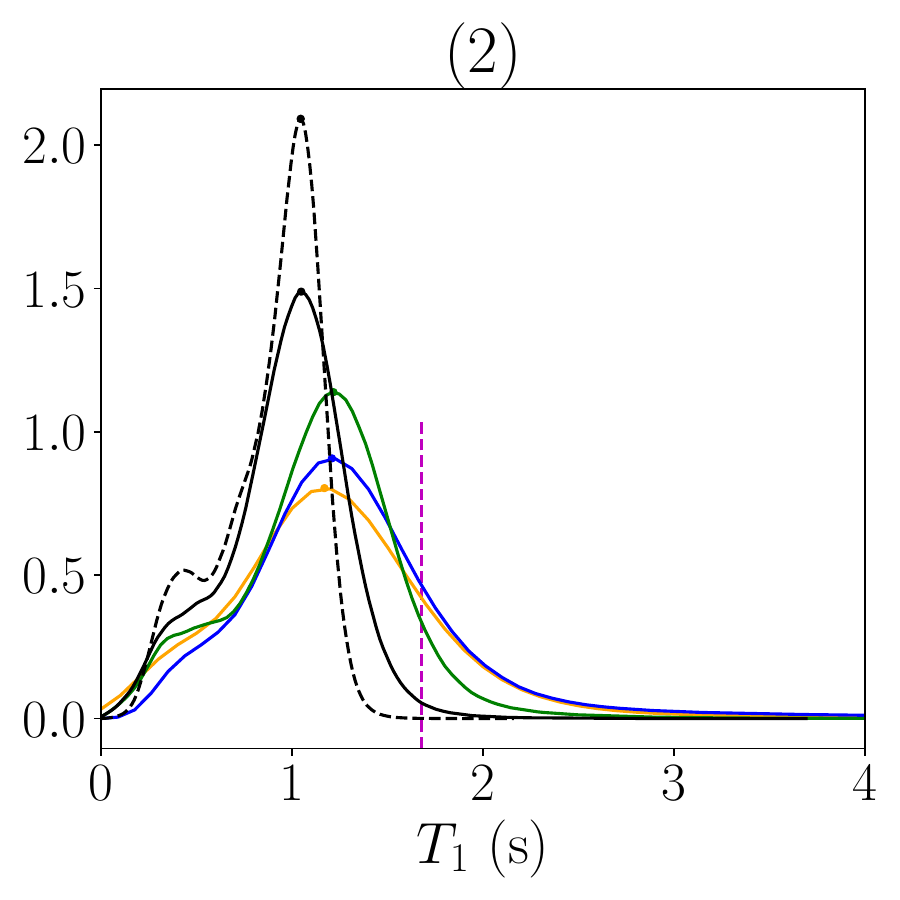}
    \includegraphics[height=0.2\linewidth]{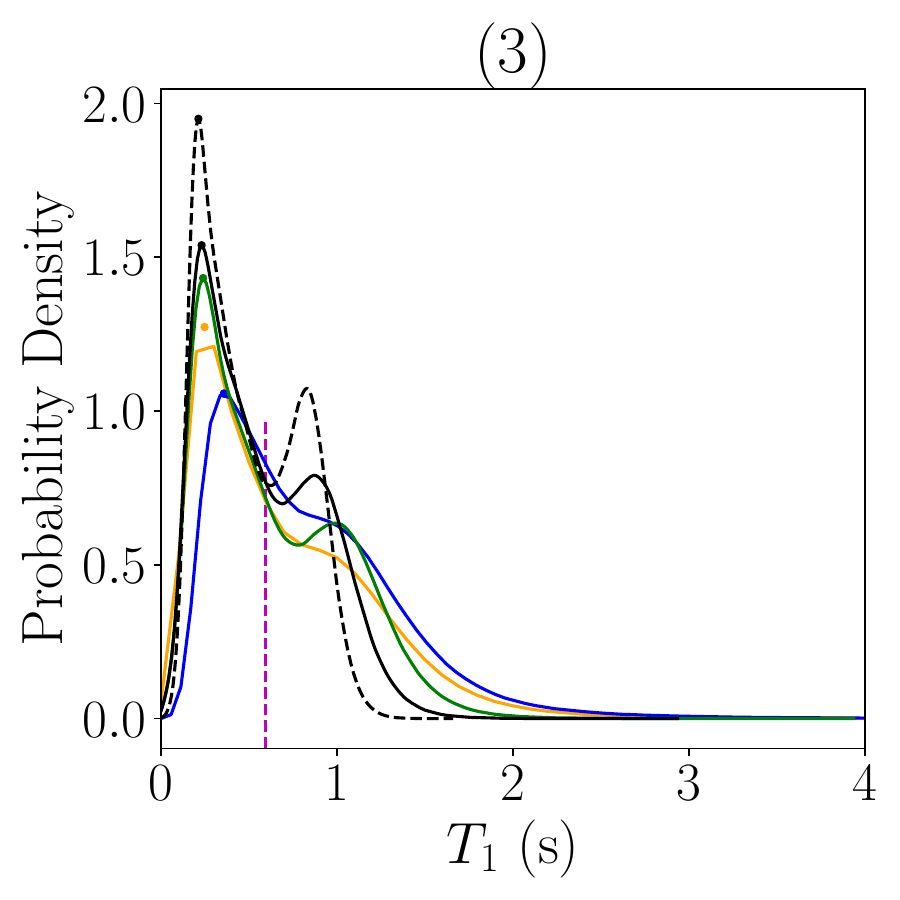}
    \includegraphics[height=0.2\linewidth]{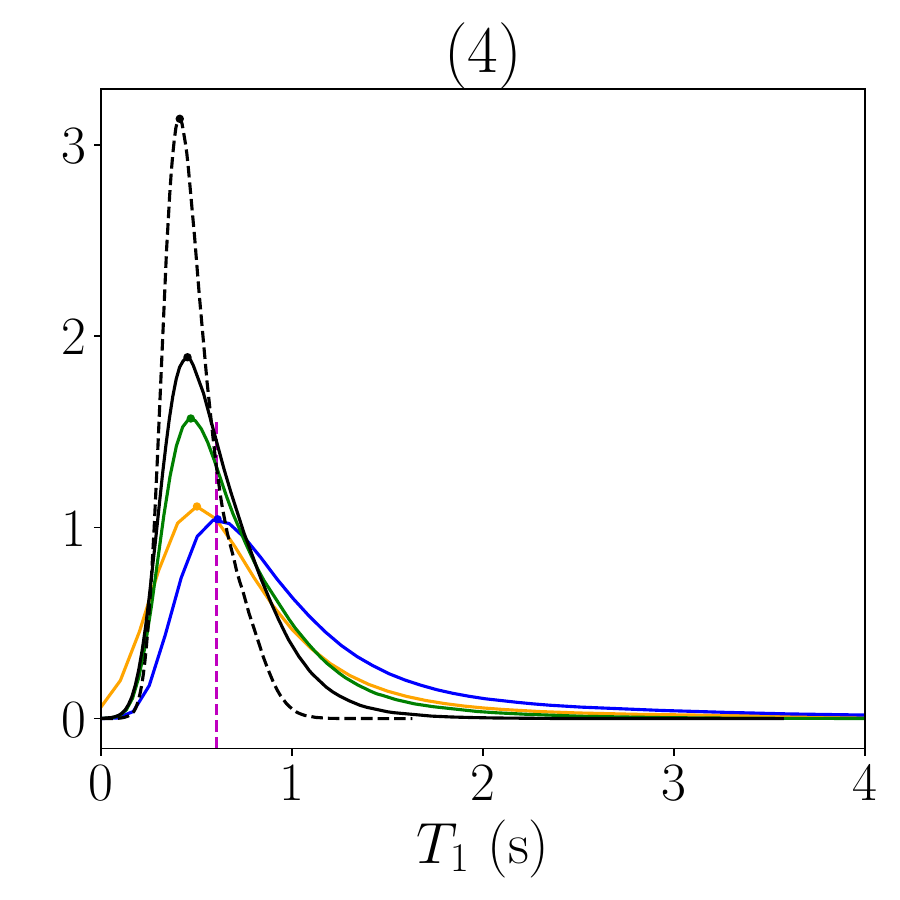}
    \includegraphics[height=0.2\linewidth]{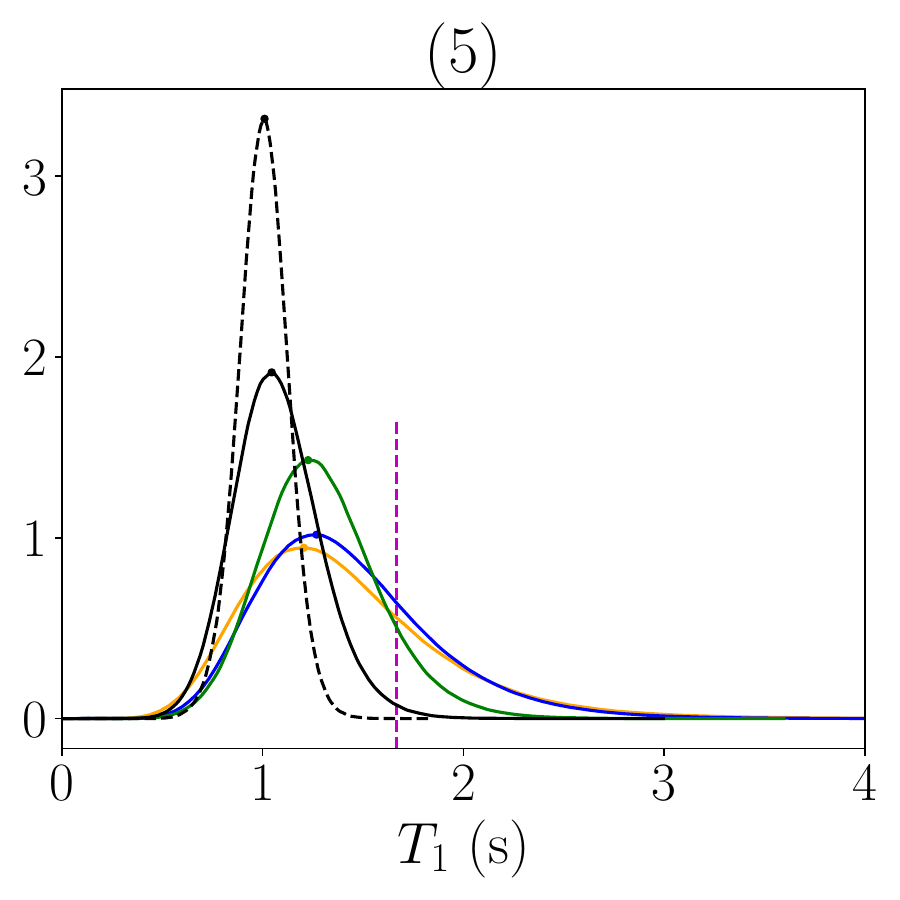}
    \includegraphics[height=0.2\linewidth]{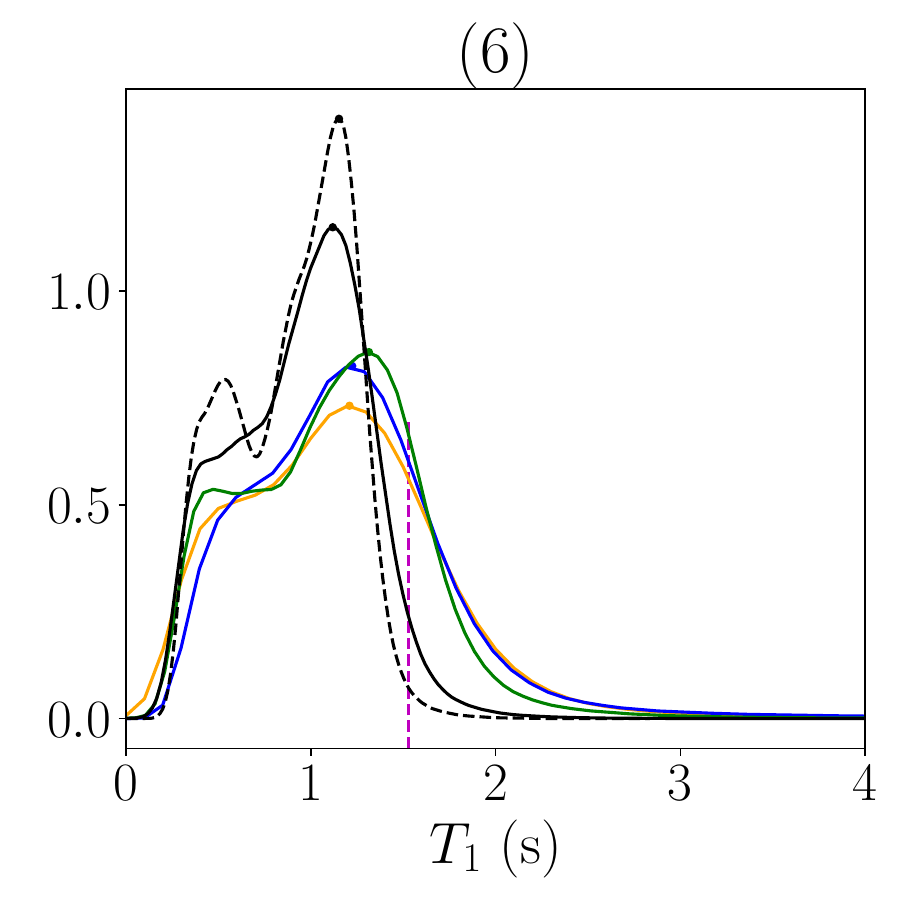}
    \caption{PDFs of $T_1$ values at selected ROIs (highlighted in red on the anatomical image, top-left). The estimated PDFs from different Bayesian methods are shown, along with the MLEs, for comparison.}
    \label{fig:pdfbreast}
\end{figure*}

\begin{table*}[]
    \centering
    \resizebox{0.97\textwidth}{!}{%
        \begin{tabular}{lcccccc}
            \toprule
            \textbf{Patch} & \textbf{MLE} & $\mathbf{U(0,50)}$ & $\mathbf{Gamma(3,1)}$ & \textbf{Bounded}~$\boldsymbol{\TV}$ & \textbf{Hierarch.}~$\boldsymbol{\TV}^1_\mu$ & $\boldsymbol{\TV}^1_\mu$ \\
            \midrule
            \textbf{Patch 1} 
            & 2.17 
            & 1.52, 2.18 $\pm$ 4.83 
            & 1.48, 1.77 $\pm$ 1.63 
            & 1.53, 1.67 $\pm$ 1.11
            & 1.25, 1.24 $\pm$ \textbf{0.34}
            & 1.27, 1.30 $\pm$ 0.53 \\
            
            \textbf{Patch 2} 
            & 1.68 
            & 1.17, 1.61 $\pm$ 4.40 
            & 1.21, 1.38 $\pm$ 1.60 
            & 1.22, 1.14 $\pm$ 0.87
            & 1.04, 0.88 $\pm$ \textbf{0.53}
            & 1.05, 0.96 $\pm$ 0.64 \\
            
            \textbf{Patch 3} 
            & 0.59 
            & 0.25, 0.71 $\pm$ 1.22
            & 0.36, 0.83 $\pm$ 1.05
            & 0.24, 0.64 $\pm$ 0.80
            & 0.21, 0.51 $\pm$ \textbf{0.55}
            & 0.23, 0.57 $\pm$ 0.66 \\
            
            \textbf{Patch 4} 
            & 0.60 
            & 0.50, 1.32 $\pm$ 4.87
            & 0.61, 1.14 $\pm$ 1.87
            & 0.47, 0.72 $\pm$ 0.82
            & 0.41, 0.49 $\pm$ \textbf{0.31}
            & 0.45, 0.60 $\pm$ 0.53 \\
            
            \textbf{Patch 5} 
            & 1.67 
            & 1.20, 1.41 $\pm$ 1.00
            & 1.27, 1.42 $\pm$ 0.88
            & 1.23, 1.29 $\pm$ 0.59
            & 1.01, 1.00 $\pm$ \textbf{0.25}
            & 1.04, 1.07 $\pm$ 0.43 \\
            
            \textbf{Patch 6} 
            & 1.53 
            & 1.21, 1.35 $\pm$ 3.24
            & 1.22, 1.24 $\pm$ 1.38
            & 1.31, 1.12 $\pm$ 1.00
            & 1.15, 0.90 $\pm$ \textbf{0.63}
            & 1.12, 0.95 $\pm$ 0.72 \\
            
            \bottomrule
        \end{tabular}
    }
    \caption{Quantitative comparison of ML and Bayesian estimates for each selected ROI. For each Bayesian model, the table reports the posterior mode, mean, and uncertainty as ``$\mathrm{mode}, \mathrm{mean} \pm 2\sigma$''. All T$_1$ values are reported in seconds (s).}
    \label{tab:breast}
\end{table*}

\begin{figure*}[tb]
    \centering
    \includegraphics[width=0.999\linewidth]{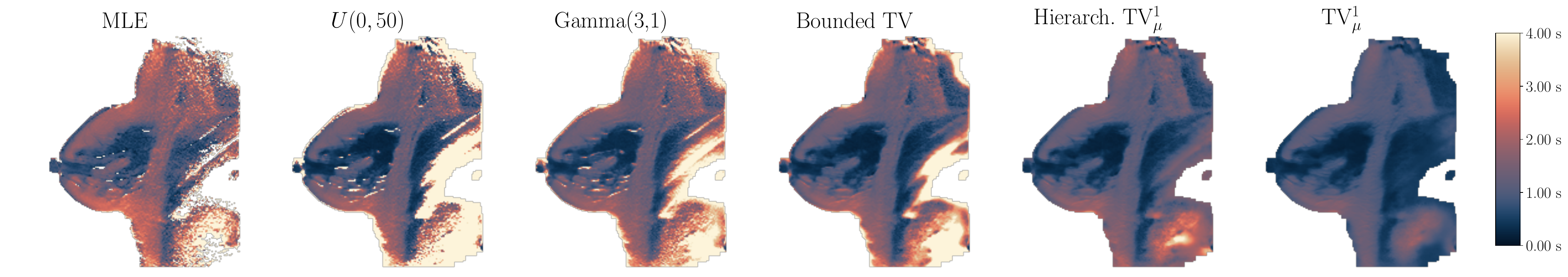}
    \caption{Comparison of mean $T_1$ maps from different methods as labeled at the top of the figure. }
    \label{fig:mapssbreast}
\end{figure*}

The $\TV^1_\mu$ and Hierarchical $\TV^1_\mu$ methods exhibit different behavior on the real dataset compared with the synthetic ones. This suggests that the choice between these methods depends on the level of prior knowledge available for a given application.
Moreover, the consistently better performance of $p=1$ compared to $p=2$ in terms of epld may be attributed to the lower sensitivity of the $\ell_1$ norm to large $T_1$ values. We also noticed that methods with the $\TV^1_\mu$ prior do not always yield the most accurate mean estimates compared to the Bounded TV method, reflecting a bias–variance trade-off.

\section{Conclusion}

This work introduces a new family of structured priors that integrate the TV function with $\ell_p$ norms, providing a proper and flexible prior distribution for Bayesian $T_1$ mapping. Unlike earlier TV-based formulations, the proposed construction preserves the desirable spatial regularization properties of TV while ensuring a normalizable prior that assigns meaningful probability across the full parameter space. Embedded within a Bayesian framework and combined with NUTS sampling, this formulation enables both stable parameter estimation and rigorous uncertainty quantification.

Across synthetic brain, cardiac, and real in-vivo breast datasets, the TV--$\ell_p$ prior consistently produced more concentrated posterior densities than the other methods, leading to reduced uncertainty, lower variance, and only a small negative bias. As a result, the parameter maps exhibit enhanced spatial coherence and more reliable credibility assessments, even under challenging noise conditions and limited data.

Overall, the results demonstrate that incorporating TV-based spatial structure together with $\ell_p$ norms within the Bayesian framework provides a robust foundation for accurate $T_1$ estimation and uncertainty quantification, and can be naturally extended to other qMRI modalities or to joint reconstruction–estimation strategies in future work. While our experiments focus on medical imaging, the method is not limited to this domain and can be applied to a wide range of problems that require structured signal recovery or spatially informed priors.


\acks{This work was supported by the Swedish Research Council (Vetenskapsrådet; grant number 2021-04810) and Lion's Cancer Research Foundation in Northern Sweden (grant numbers LP 22-2319, LP 24-2367, and AMP 26-1265).}

%
\ethics{The work followed appropriate ethical standards in conducting research and writing the manuscript, following all applicable laws and regulations regarding treatment of animals or human subjects.}

\coi{We declare we don't have conflicts of interest.}

\data{All data used in this study are either publicly available or were synthetically generated as described in the manuscript. 
The synthetic brain and cardiac $T_1$ mapping datasets can be fully reproduced using the simulation procedures provided in the Data section. 
The real breast MRI data were obtained from the publicly accessible \textit{QIN-BREAST-02} collection available through The Cancer Imaging Archive (TCIA). 
No additional restrictions apply to the availability or reproduction of the datasets used in this work. The code associated with this work is publicly available at: \url{https://github.com/Disi2022/Proper-TV-Structured-Priors}}

\bibliography{sample}

\end{document}